\documentclass[11pt]{article} %
\usepackage[preprint]{acl}

\usepackage{times}
\usepackage{latexsym}

\usepackage[T1]{fontenc}

\usepackage[utf8]{inputenc}

\usepackage{inconsolata}

\usepackage{amsmath,amsfonts,bm}

\def\eqref#1{equation~\ref{#1}}

\def\1{\bm{1}}

\DeclareMathAlphabet{\mathsfit}{\encodingdefault}{\sfdefault}{m}{sl}
\SetMathAlphabet{\mathsfit}{bold}{\encodingdefault}{\sfdefault}{bx}{n}

\usepackage{microtype}
\usepackage{hyperref}
\usepackage{url}
\usepackage{booktabs}
\usepackage{graphicx}
\usepackage{subcaption}
\usepackage{booktabs}       %
\usepackage{amsfonts}       %
\usepackage{amsthm}
\usepackage{nicefrac}       %
\usepackage{xcolor}  
\usepackage{caption} 
\usepackage{wrapfig}
\usepackage{IEEEtrantools}
\usepackage{amsmath}
\usepackage{wrapfig}
\usepackage{lipsum}
\usepackage{placeins}
\usepackage{array}
\usepackage{IEEEtrantools}
\definecolor{darkblue}{rgb}{0, 0, 0.5}
\hypersetup{colorlinks=true, citecolor=darkblue, linkcolor=darkblue, urlcolor=darkblue}

\title{
Disentangling the Roles of Representation and Selection in Data Pruning}

\author{Yupei Du, Yingjin Song, Hugh Mee Wong, Daniil Ignatev, Albert Gatt, Dong Nguyen\\ 
Utrecht University, The Netherlands \\ 
\texttt{\{y.du, y.song5, h.m.wong, d.ignatev, a.gatt, d.p.nguyen\}@uu.nl} 
}

\newtheorem{theorem}{Theorem}[section]

\newtheorem{remark}[theorem]{Remark}
\newtheorem{corollary}[theorem]{Corollary}
\newtheorem{definition}[theorem]{Definition}

\begin{document}

\maketitle

\begin{abstract}
Data pruning—selecting small but impactful subsets—offers a promising way 
to efficiently scale NLP model training. 
However, existing methods often involve many different design choices, 
which have not been systematically studied. 
This limits future developments. 
In this work, we decompose data pruning into two key components:
the \textit{data representation} and  the \textit{selection algorithm}, 
and we systematically analyze their influence on the selection of instances. 
Our theoretical and empirical results highlight the crucial role of representations: 
better representations, e.g., training gradients, 
generally lead to a better selection of instances, regardless of the chosen selection algorithm.
Furthermore, different selection algorithms excel in different settings, and
none consistently outperforms the others. Moreover, 
the selection algorithms do not always align with their intended objectives: for example, 
algorithms designed for the same objective can select drastically different instances,
highlighting the need for careful evaluation.
\end{abstract}

\section{Introduction}\label{sec:intro}
A major drive of recent progress in NLP has been the scaling of training data, 
regarding both pretraining~\citep{kaplan2020scaling,hoffmann2022training,sardana2024beyond} 
and fine-tuning~\citep{zhang2024when}. 
However, recent studies have shown that 
by carefully selecting a small subset of the original dataset, 
a process known as \emph{data pruning}, 
one can train models of comparable or even better performance with much less data~\citep{
    sorscher2022beyond,du2023ftft,xia2024less}.

Different data pruning methods exist, involving various design choices. 
However, no existing work has systematically studied the influence of each choice, hindering future progress. 
Although these methods appear diverse, 
we decompose them into two key components:
(1) obtaining \emph{data representations}, usually from a \emph{reference model}, and 
(2) running a \emph{selection algorithm} using these representations. 
Moreover, while the specific steps of selection algorithms vary, 
they share common \emph{objectives}, such as maximizing the difficulty or diversity of the selected instances.
Distinguishing these two components allows us to study fundamental questions: 
\emph{which representations and selection objectives work better}, and 
\emph{whether selection algorithms indeed meet their objectives}.

In this paper, 
we conduct a comprehensive study to answer these questions 
through both a theoretical and empirical lens.
Our contributions are: 

\begin{enumerate}
    \item To study the impact of different design choices in existing data pruning methods, 
    we conduct a comprehensive review and identify two key components: 
    data representations and selection algorithms. 
    Moreover, we identify three common sources for representations: 
    \emph{training dynamics}, e.g., loss trajectory across epochs, 
    \emph{hidden states}, and \emph{gradients}; 
    and three common selection objectives: 
    maximizing \emph{difficulty}, \emph{diversity}, and \emph{relevance} to validation data 
    of the selected instances (\S\ref{sec:decoupling}). 

    \item To study which representations are more effective and why, 
    we first identify three key criteria that effective representations should satisfy. 
    We then theoretically analyze whether different representations meet these criteria. 
    Finally, on both a simple synthetic task and NLP task-specific fine-tuning, 
    we empirically validate that 
    the representations that are more useful in theory (i.e., meet more criteria), e.g., gradients, 
    are indeed more effective than others, e.g., hidden states (\S\ref{subsec:rep_reveal}).

    \item %
    We study which selection objectives are more effective 
    and find that no one clearly stands out: 
    which selection objective works better depends on the context.
    For example, maximizing relevance to validation data excels 
    when a substantial train-test distribution shift is present, 
    and maximizing difficulty works well with high data budgets. 
    Surprisingly, 
    \emph{representations are more influential than selection algorithms}: 
    when different selection algorithms use the same representation, 
    the overlap in selected instances is greater than 
    when the same selection algorithm is used with different representations 
    (\S\ref{subsec:properties_selection}).

    \item To gain insights into whether algorithms follow their intended objectives, 
    we visualize their instance selection, and assess the consistency 
    between the selections of different algorithms aiming for difficulty. 
    Surprisingly, our results suggest that these algorithms do not always align with their objective. 
    For instance, when maximizing difficulty, 
    they sometimes prefer instances that are correctly predicted and far from the decision boundary; 
    however, these are usually considered to be easier instances.
    Furthermore, the same objective of maximizing difficulty 
    can lead to drastically different selections (\S\ref{sec:empirical_study}).

\end{enumerate}

Our findings provide actionable insights for the development of data pruning methods. 
Future research should:
(1) develop scalable yet strong representations, and
(2) carefully assess whether selection algorithms follow their intended objective.\footnote{
    Our code is available at \url{https://github.com/nlpsoc/data_pruning_disentangle}.}

\section{Representation-selection decoupling}\label{sec:decoupling}

Given a \emph{data budget}, such as 30\% of the original dataset, 
data pruning methods aim to select an informative subset of the data. 
However, it remains unclear 
how different design choices of these methods impact their effectiveness, 
because previous studies typically treat them as cohesive units. 
To address this gap, we identify two key components in data pruning methods: 
first, obtaining \emph{representations} for each instance, 
such as hidden states or gradients, using a \emph{reference model}, 
either off-the-shelf or fine-tuned on the original dataset;
second, a \emph{selection algorithm} to choose a subset of the data 
guided by a \emph{selection objective}, 
such as maximizing the difficulty of the selected instances.
This selected subset is then used to train the \textit{main model}, 
which is the final model of interest.\footnote{
    We exclude methods that rely on prompting LLMs for quality scoring~\citep{
        sachdeva2024train,chen2024alpagasus,lu2024instag,liu2024what}, 
        as these approaches add complexity through heuristic prompts and often function as black boxes, 
        making their results difficult to interpret.}

\subsection{Commonly-used representations and selection objectives}\label{sec:common_choices}

\paragraph{Representations}

\emph{Training dynamics} are widely used as a source for extracting representations, 
especially in fine-tuning tasks. 
These include metrics such as the correctness of predictions across epochs~\citep{toneva2018an}, 
prediction probabilities of the correct class~\citep{swayamdipta-etal-2020-dataset,pmlr-v139-jiang21k}, 
training error norms~\citep{paul2021deep}, 
the number of layers required for correct classification~\citep{baldock2021deep}, 
and perplexity~\citep{moore-lewis-2010-intelligent,marion2023more,kwok2024dataset}.
Differently, \emph{hidden states} from pretrained language models 
are frequently used in pretraining scenarios~\citep{
    abbas2023semdedup,NEURIPS2023_a8f8cbd7},
because they can capture semantic information while being computationally efficient.
\emph{Gradients} are another common representation in fine-tuning. 
They are often used to estimate the impact of specific instances on model predictions, 
either through influence functions~\citep{pmlr-v70-koh17a,park2023trak} 
or training unrolling methods~\citep{NEURIPS2020_e6385d39,xia2024less}.
There are also a few methods that use text-based features, 
such as bag-of-words~\citep{canuto2018thorough}.

\paragraph{Selection objectives}

After obtaining representations, 
various objectives are used to guide the implementation of selection algorithms. 
One common objective is to maximize the \emph{difficulty} of selected instances, 
i.e., to select those that are harder for models to fit, as indicated by 
being more forgettable~\citep{toneva2018an}, 
having a lower prediction confidence~\citep{swayamdipta-etal-2020-dataset},
a higher loss~\citep{pmlr-v139-jiang21k,li-etal-2024-quantity}, 
more layers required for prediction~\citep{baldock2021deep},
a higher perplexity~\citep{kwok2024dataset}, 
a higher self-influence~\citep{thakkar-etal-2023-self}, 
and larger distances from prototypical examples~\citep{sorscher2022beyond}.
Another objective is to maximize \emph{diversity} in the selected data~\citep{
    carbonera2015density,carbonera2016novel,malhat2020new}.
For example, \citet{abbas2023semdedup} measure the similarity between instances 
and keep only one from each pair of highly similar instances, 
and \citet{Yang2024SmallToLargeS} randomly sample from different clusters of instances. 
Moreover, when specializing models, e.g., adapting a general model to the medical domain, 
it is common to maximize the \emph{relevance} of selected instances to validation data. 
For example, assuming the availability of a validation set, 
\citet{xia2024less} and \citet{engstrom2024dsdm} select 
the most influential training instances 
based on training unrolling methods and influence functions, respectively.

\subsection{Representative methods}\label{subsec:representative_methods}

\begin{table*}[h]
    \centering
    \renewcommand{\arraystretch}{0.9} %
    \resizebox{\textwidth}{!}{ %
        \small
        \begin{tabular}{@{} c m{11em} m{11em} m{12em} @{}}
            \toprule
            & \multicolumn{3}{c}{\textbf{Representations}} \\
            \cmidrule(lr){2-4}
            \textbf{Selection} & \textbf{Training dynamics} & \textbf{Hidden states} & \textbf{Gradients} \\
            \midrule
            Max. diversity & SmallToLarge (S2L) \newline \citep{Yang2024SmallToLargeS} & SemDedup \newline \citep{abbas2023semdedup} &  \\ 
            \midrule
            Max. difficulty & Hard-to-learn \newline \citep{
            swayamdipta-etal-2020-dataset, pmlr-v139-jiang21k, ince2023harnessing} & Prototypicality \newline \citep{sorscher2022beyond} \newline SemDedup \newline \citep{abbas2023semdedup} & Self-Influence (SI) \newline \citep{NEURIPS2020_1e14bfe2,bejan-etal-2023-make} \\ 
            \midrule
            Max. relevance &  &  & LESS \citep{xia2024less} \\
            \bottomrule
        \end{tabular}
    }
    \caption{Representative methods from \S\ref{subsec:representative_methods}, 
    categorized by their representations and selection objectives.}
    \label{tab:representative_methods}
\end{table*}

Having identified commonly used representations and selection objectives,
we focus on six representative methods (see Table~\ref{tab:representative_methods}, 
where we also include their corresponding representations and objectives).
These methods cover the most common representation types: 
training dynamics, hidden states, and gradients, 
as well as key selection objectives: 
maximizing difficulty, diversity, and relevance to validation data.
We provide an in-depth description of the methods in Appendix~\ref{app_sec:overview_dpm}.\footnote{
    Different methods were originally proposed for specific contexts. 
    \textbf{Our goal is not to invalidate them}, 
    but to offer additional insights into their components. 
}

\section{Deciphering the impact of different components on instance selection}\label{sec:deciphering} 

This section investigates how 
different data representations and selection algorithms 
influence the selection of training instances.
We first provide a theoretical analysis on  
how various representations differ in terms of the signals they encode 
(\S\ref{subsec:rep_reveal}). 
Clarifying these differences allows us to 
understand the fundamental benefits and limitations of different representations, 
without considering specific selection algorithms. 
Next, we empirically compare the instances selected with
different combinations of representation and selection algorithms, 
through experiments with both an interpretable classifier on a synthetic dataset, 
and the fine-tuning of language models for various tasks 
(\S\ref{subsec:properties_selection}).

\paragraph{Notation}

We denote the original training set with $N$ instances as 
$\mathcal{D} = \{(x_i, y_i)\}_{i=1}^N$. 
The selected subset of data is represented by $\mathcal{S} \subset \mathcal{D}$. 
We use $B$ to denote the data budget (e.g., $B = |\mathcal{S}| = 0.2N$). 
For a data point $(x_i, y_i)$ and a model $\mathcal{M}$, 
we use $p_{\mathcal{M}}(y_i|x_i)$ and $\ell_{\mathcal{M}}(x_i, y_i)$ to denote 
the model's prediction probability of the correct class/token and loss, 
$h_{\mathcal{M}}(x_i)$ to denote the last hidden state, 
i.e., before classifier or unembedding layer of $\mathcal{M}$. 
Moreover, we use $\nabla_\theta \ell_{\mathcal{M}}(x_i, y_i)$ to represent 
the gradient of a group of parameters $\theta$ of $\mathcal{M}$ w.r.t. $\ell_{\mathcal{M}}(x_i, y_i)$. 
When taking the training process of $T$ epochs into account,
we use $p_{\mathcal{M}}^{(t)}(y_i|x_i)$ and $\nabla_{\theta} \ell_{\mathcal{M}}^{(t)}(x_i, y_i)$
to denote $p_{\mathcal{M}}(y_i|x_i)$ and $\nabla_{\theta} \ell_{\mathcal{M}}(x_i, y_i)$ 
at epoch $t \leq T$, 
and $\eta_t$ to denote the average learning rate of the model in epoch $t$.

\subsection{Properties of representations}\label{subsec:rep_reveal}

Our analysis centers on one key quantity: 
the \emph{distance} between two instances, $i$ and $j$, based on different representations. 
Measuring instance distances allows for assessing 
how well representations group instances of shared attributes together. 
Indeed, these distances are central to most methods, enabling functionalities like 
clustering (e.g., S2L and Prototypicality), duplicate identification (e.g., SemDeDup), 
and instance relevance measurement (e.g., LESS). 
For example, 
the Prototypicality method first clusters instances, and then 
selects instances far from the centroids as ``difficult" examples. 
However, this raises an important question: 
\emph{are the clusters good enough in separating different instances?}

Specifically, we first investigate the criteria for good representations, 
including the \emph{information} representations should encode, 
and their \emph{discriminative power} for different inputs, 
i.e., how well they separate different types of instances. 
We then analyze whether different representations meet these criteria. 
To the best of our knowledge, we are the first to 
systematically compare different representations in the context of data pruning.

\paragraph{Setup}

For the simplicity of analysis,\footnote{
    Our analysis can be extended to multi-class classification or generation tasks 
    by considering the prediction of a specific class or token, 
    similar to~\citet{park2023trak}. 
} 
we consider a binary classification task with labels $y \in \{-1, +1\}$, 
optimized with binary cross-entropy loss, 
and focus on the \textbf{classification layer}. 
Formally, given two training instances $(x_i, y_i)$ and $(x_j, y_j)$, 
a model $\mathcal{M}$, and its classification layer $w$,
we study the squared Euclidean distances between these two instances, computed by 
hidden states $h_{\mathcal{M}}(\cdot)$, losses $\ell_{\mathcal{M}}(\cdot)$, 
and gradients $\nabla_w \ell_{\mathcal{M}}(\cdot)$. 
We denote them as $D_h$, $D_{\ell}$, and $D_g$ respectively.
Moreover, we use \textbf{hidden states} as the basis for our analysis, 
because they serve as inputs to the classification layer to compute other representations.\footnote{
    In other words, we treat hidden states as inputs throughout the analysis, 
    and compute quantities such as the distances between instances and the decision boundary based on them.
}

\paragraph{What makes a good representation?}

To select a minimal subset of training data while preserving generalization, 
we propose that the selections need to be 
\emph{(1) non-redundant}, ensuring minimality, and 
\emph{(2) diverse}, promoting generalization. 
Importantly, we argue that the redundancy and diversity here should be considered  
with respect to model training. 
Specifically, 
selected instances must be diverse enough to train a robust classifier $w$, 
while discarding less relevant examples. 
This entails retaining instances 
(1) \emph{close to the decision boundary}\footnote{
    The decision boundary of the final model, 
    which we approximate using that of the reference model. 
}, as those far away are 
either trivial (following the representer theorem~\citep{NEURIPS2018_8a7129b8}), 
or are mislabeled or rare outliers that destabilize training~\citep{2022PrioritizedTraining}: 
this helps select non-redundant instances; 
and (2) \emph{diverse} enough, 
as otherwise we are likely to obtain biased models, 
i.e., models that make predictions using a narrow set of rules~\citep{NEURIPS2023_a8f8cbd7}. 

We therefore argue that good representations should 
ensure that the distances between instances $D$ satisfy three key criteria.
First, $D$ should \emph{account for instances' distances to the decision boundary}.
Second, $D$ should \emph{contain instance label information}, 
to help selection algorithms balance samples across different labels.
Third, $D$ should \emph{be more discriminative for important instances}, 
e.g., those are closer to the decision boundary. 
This enables selection algorithms to preserve diversity among these important samples, 
by identifying their differences, while deprioritizing less relevant data.

\paragraph{Encoded information} 

We express gradients and losses as functions of hidden states and model parameters 
to study information encoded by different representations, 
and have the following result, for which the derivation can be found in Appendix~\ref{app_sec:proofs}.

\begin{remark}[Explicit expressions]{\label{remark:explicit}}
    Let $z_{*} = y_{*} w^T h_{\mathcal{M}}(x_{*})$ be the (signed and scaled) distance 
    from $h_{\mathcal{M}}(x_{*})$ to the decision boundary.
    We have
    $D_{\ell} = (\log{((1 + e^{-z_i}) / (1 + e^{-z_j})}))^2$, 
    $
    D_g = \|\frac{y_i h_{\mathcal{M}}(x_i)}{1+e^{z_i}} - \frac{y_j h_{\mathcal{M}}(x_j)}{1+e^{z_j}}\|_2^2
    = \frac{\|h_{\mathcal{M}}(x_i)\|_2^2}{(1+e^{z_i})^2} + \frac{\|h_{\mathcal{M}}(x_j)\|_2^2}{(1+e^{z_j})^2}
    - 2 \frac{y_i y_j h_{\mathcal{M}}(x_i)^T h_{\mathcal{M}}(x_j)}{(1+e^{z_i})(1+e^{z_j})}
    $.
\end{remark}

We make two key observations. 
First, compared to the distance between hidden states ($D_{h}$), 
the distance between losses ($D_{\ell}$) additionally 
integrates the distances to the decision boundary (i.e., $z_i$ and $z_j$). 
Second, gradients ($D_g$) further reflect label agreement. 
Specifically, $D_g$ is small when 
(1) instances are easy (i.e., $z_i$ and $z_j$ are large), increasing denominators; 
and (2) hidden states are similar when their labels agree, and vice versa, 
increasing the third term's numerator. 
These observations show that 
losses and gradients are stronger than hidden states, 
for identifying instances that are similar for the training process, 
because they encode instances' distances to the decision boundary. 
Furthermore, only gradients are label-aware. 

\paragraph{Discriminative power}

We examine the \emph{discriminative power} of the distance between two instances 
based on different representations: 
the more sensitive these distances are to the changes of inputs, i.e., hidden states here, 
the more discriminative they are. 
Specifically, we analyze how this discriminative power varies 
with an instance’s \emph{distance to the decision boundary}. 
Ideally, distances should be more discriminative for instances near the decision boundary, 
enabling data pruning methods to capture finer distinctions, 
while ignoring variations among instances further away, since they are less relevant.
To quantify this, 
we measure the \textbf{Jacobian magnitudes} of these distances w.r.t. hidden states.

Formally, let 
$J_{h_{\mathcal{M}}(x_i)}(D_h)$, $J_{h_{\mathcal{M}}(x_i)}(D_{\ell})$, 
and $J_{h_{\mathcal{M}}(x_i)}(D_g)$ 
be the Jacobian matrices of $D_h$, $D_{\ell}$, and $D_g$  
with respect to $h_{\mathcal{M}}(x_i)$. 
By our distance definition, for a given representation \(r(\cdot)\), 
the distance between two instances is defined as 
$D_r = \|r(x_i, y_i) - r(x_j, y_j)\|_2^2$.
We can then write the Jacobian of \(D_r\) with respect to \(h_{\mathcal{M}}(x_i)\) as
$$
\begin{aligned}
&J_{h_{\mathcal{M}}(x_i)}(D_r) 
=  \frac{\partial D_r}{\partial r(x_i, y_i)} 
   \frac{\partial r(x_i, y_i)}{\partial h_{\mathcal{M}}(x_i)} \\
&= 2 J_{h_{\mathcal{M}}(x_i)}\bigl(r(x_i, y_i)\bigr)^\top 
\bigl(r(x_i, y_i) - r(x_j, y_j)\bigr).
\end{aligned}
$$

Here we can see the Jacobian is influenced by the distance value through $r(x_i, y_i) - r(x_j, y_j)$.
However, our goal here is to quantify 
how inputs' distances to the decision boundary (based on hidden states) 
relate to $D_r$'s discriminative power, independent of the specific distance value. 
Therefore, \emph{we focus on $\mathbf{J_{h_{\mathcal{M}}(x_i)}\bigl(r(x_i, y_i)\bigr)}$}, 
and use the \textbf{spectral norm} to measure its magnitude. Formally:

\begin{definition}[Discriminative power]
    We define the discriminative power of losses and gradients 
    as the spectral norms of their Jacobian w.r.t. $h_{\mathcal{M}}(x_i)$:
    \begin{align*}
        \mathcal{C}_{\ell} &= \|J_{h_{\mathcal{M}}(x_i)}(\ell_{\mathcal{M}} (x_i, y_i))\|, \\
        \mathcal{C}_g &= \|J_{h_{\mathcal{M}}(x_i)}(\nabla_w \ell_{\mathcal{M}}(x_i, y_i))\|, 
    \end{align*}
    where $\|\cdot\|$ denotes the spectral norm.
    Analogously, we get $\mathcal{C}_h=1$. 
\end{definition}

Based on the above definitions we have the following results. 
Both proofs are in Appendix~\ref{app_sec:proofs}. 
\begin{theorem}[Region dependence]{\label{theorem:discriminative}}
    $\mathcal C_{\ell}$ and $\mathcal C_g$ are dependent on 
    inputs' distances to the decision boundary, satisfying
    $$
    \begin{aligned}
    \mathcal{C}_{\ell} &= \frac{\|w\|}{1+e^{z_i}} 
    = (1 - p_{\mathcal{M}}(y_i|x_i)) \|w\|, \text{ and } \\
    \mathcal{C}_g &\leq 
    \frac{1}{1+e^{z_i}}+\frac{e^{z_i}}{\left(1+e^{z_i}\right)^2}\|h_{\mathcal{M}}(x_i)\| \|w\|.
    \end{aligned}
    $$
\end{theorem}

\begin{figure}[t]
    \centering
    \includegraphics[width=0.45\textwidth]{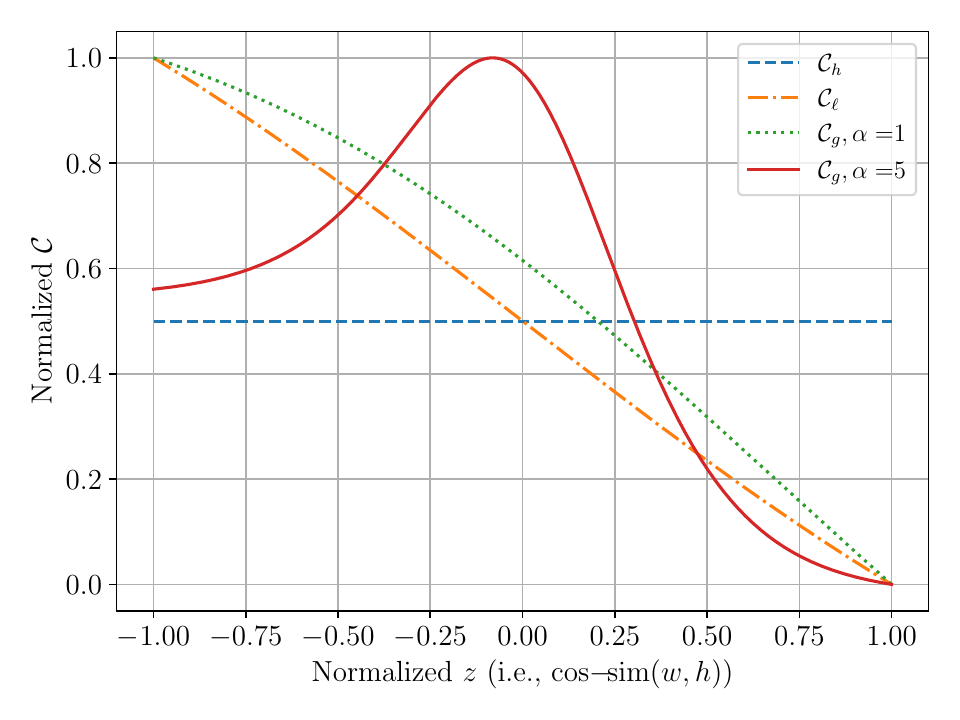}
    \caption{
        (Min-max normalized) discriminative power of the distance between instances, 
        computed by different representations: 
        the loss's discriminative power ($\mathcal{C}_{\ell}$) monotonically decreases, 
        while the gradient's ($\mathcal{C}_g$) 
        peaks near the decision boundary for large $\alpha$s. 
    }
    \label{fig:discriminative-power}
\end{figure}

\begin{corollary}{\label{corollary:gradient_peak}}
    Let $\alpha := \|w\| \|h_{\mathcal{M}}(x_i)\|$. 
    When $\alpha$ is smaller than the positive root of 
    $-x (1-e^{x}) = 1 + e^{x}$ (approximately 1.544), 
    $\mathcal{C}_g$ decreases monotonically as $z_i$ increases, 
    similar to $\mathcal{C}_{\ell}$.
    However, when $\alpha$ is larger, 
    $\mathcal{C}_g$ increases with $z_i$ for $z_i \leq \log\left(\frac{\alpha-1}{\alpha+1}\right)$, 
    and decreases for $z_i > \log\left(\frac{\alpha-1}{\alpha+1}\right)$.
\end{corollary}

\begin{remark}
    Theorem~\ref{theorem:discriminative} shows that, 
    $\mathcal{C}_{\ell}$ monotonically decreases with $z_i$ and $p_{\mathcal{M}}(y_i|x_i)$, 
    i.e., the prediction probability. 
\end{remark}

\begin{remark}
    Corollary~\ref{corollary:gradient_peak} indicates that, 
    when $\alpha>\sim 1.544$, $\mathcal C_g$ peaks at 
    $z_i = \log\left(\frac{\alpha-1}{\alpha+1}\right)$, 
    which means the corresponding data point is close to the decision boundary but misclassified.
    Meanwhile, when $\alpha$ is smaller, $\mathcal C_g$ decreases with $z$ 
    (and thus prediction probability), 
    similar to $\mathcal{C}_{\ell}$.\footnote{
        Intuitively, $\alpha$ reflects the magnitudes of 
        the model's weights and the inputs' hidden states. 
        Across all models of our experiments, we consistently find $\alpha>\sim 1.544$. 
    }
\end{remark}

We provide a visual illustration of the discriminative power 
of different representations in Figure~\ref{fig:discriminative-power}.
In particular, 
we highlight the property of the gradients with $\alpha = 5$: 
it is discriminative when predictions are wrong, 
and peaks near the decision boundary, 
and becomes very small once the prediction is confidently correct.
According to our previous analysis, 
\emph{this effectively enables the selection of diverse and non-redundant examples} 
for learning classifiers, 
as they make instances near the decision boundary more distinguishable while ignoring the redundant easy ones. 
In contrast, algorithms that use losses 
will likely over-select those with a high loss (potentially destabilizing training),
while distances based on hidden states are indifferent to inputs' distance to the decision boundary.

\subsection{Properties of selection algorithms}{\label{subsec:properties_selection}}

\begin{figure*}[t]
    \centering
    \begin{subfigure}[t]{0.54\textwidth}
        \centering
        \includegraphics[width=\linewidth]{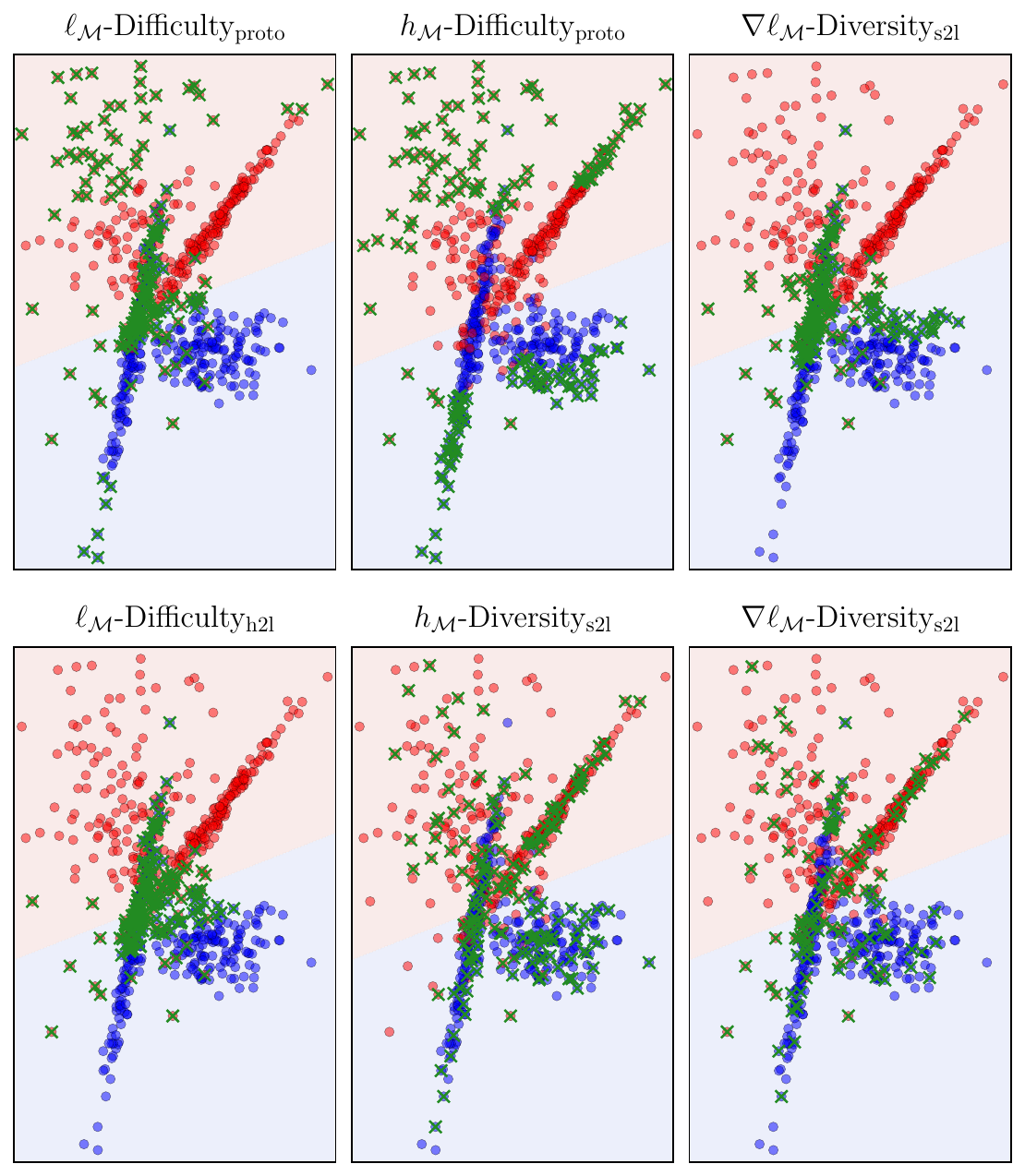}
        \caption{Selections on synthetic data.}
        \label{fig:selection_synthetic}
    \end{subfigure}%
    \begin{subfigure}[t]{0.46\textwidth}
        \centering
        \includegraphics[width=\linewidth]{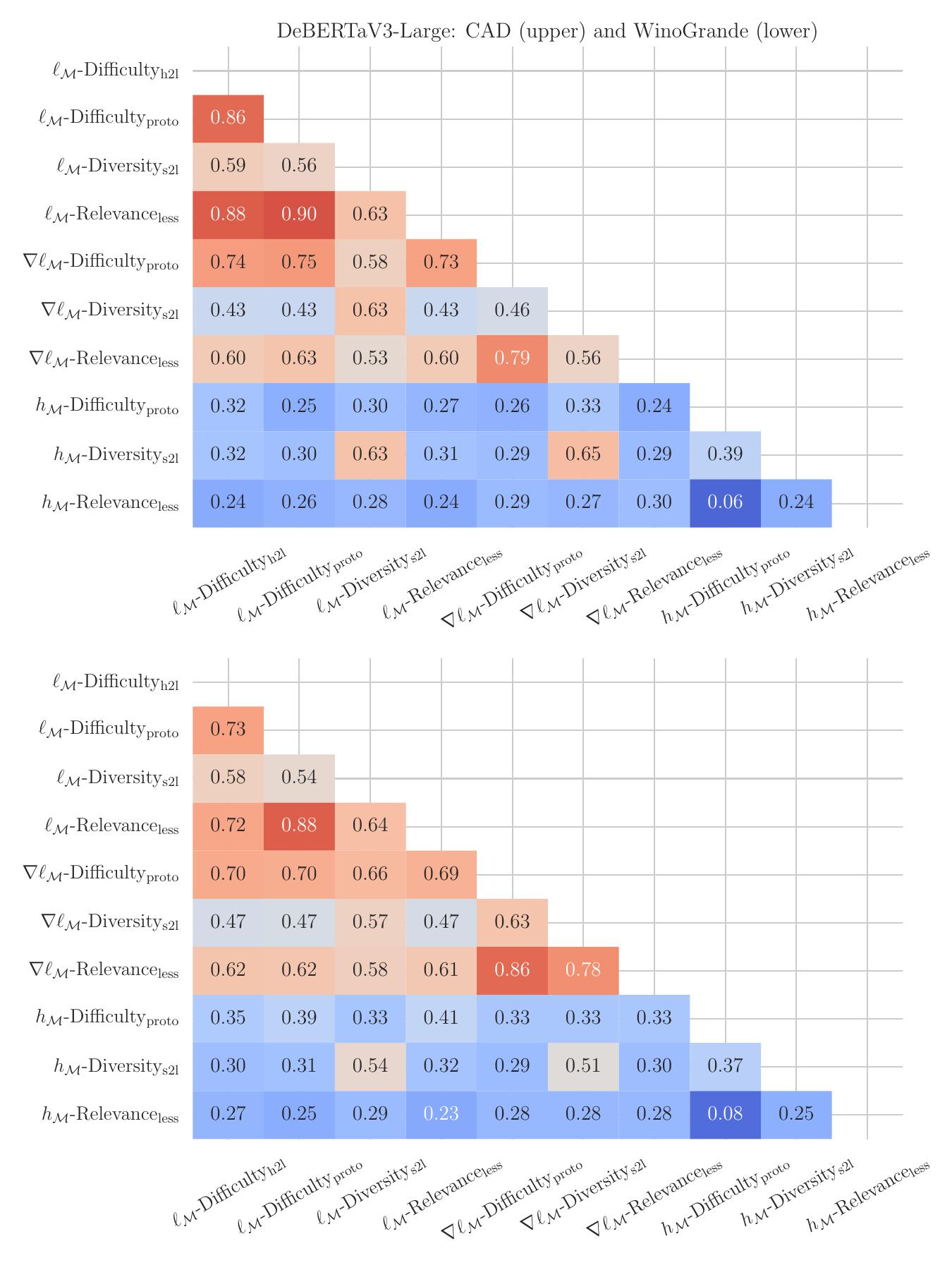}
        \caption{Selection consistency on NLP tasks.}
        \label{fig:selection_nlp}
    \end{subfigure}
    \caption{
        The consistency of selections 
        across different representations and selection algorithms.
        (a) Synthetic data: 
        we generate 600 data points from a 2D Gaussian mixture model 
        with red and blue data points representing two classes.
        The different background colors visualize the decision boundary 
        of the logistic regression reference model.
        The green Xs are the selected data points (30\% of the data).
        (b) NLP tasks:
        we use different methods to select 30\% of the data points 
        and compute their overlapping ratios 
        (i.e., $|\mathcal{S}_1 \cap \mathcal{S}_2| / |\mathcal{S}_1|$ 
        for two subsets $\mathcal{S}_1$ and $\mathcal{S}_2$). 
        Here we show the results for DeBERTaV3-Large on CAD and WinoGrande.
    }
    \label{fig:selection}
\end{figure*}

Building on the insights into representations (\S\ref{subsec:rep_reveal}), 
this section examines the properties of selection algorithms. 
We focus on two key aspects. 
First, we analyze 
\emph{how changing the data representations affect the selection of instances}, 
to understand the joint effects of both steps 
and the sensitivity of selection algorithms to different representations.
Second, we investigate 
\emph{whether selection algorithms indeed follow their objectives}, 
by visualizing the selections 
and comparing the overlap between selection algorithms with the same objective.

\paragraph{Setup} 
We focus on three selection algorithms with different objectives 
from \S\ref{subsec:representative_methods}: 
(1) prioritizing difficulty, as in prototypicality (\textbf{difficulty\(_{\text{proto}}\)}); 
(2) prioritizing diversity, as in S2L (\textbf{diversity\(_{\text{s2l}}\)}); and 
(3) prioritizing relevance to validation data, as in LESS (\textbf{relevance\(_{\text{less}}\)}).
We combine each selection algorithm with all three representations.
We also compare Hard-to-Learn (\textbf{difficulty\(_{\text{htl}}\)}) with prototypicality, 
because both methods aim to select \emph{difficult} instances. 
We use 30\% of the data as our budget. 

First, we conduct a synthetic experiment to provide an interpretable analysis. 
We use a 2D Gaussian mixture model to generate 600 data points, 
which we treat as the hidden states. 
We then train a logistic regression classifier as the reference model 
to collect training dynamics and gradients.
We visualize the selected instances in Figure~\ref{fig:selection_synthetic}.
 
Second, we conduct task-specific fine-tuning experiments 
on three different types of tasks: 
CAD (binary hate speech classification, \citealp{vidgen-etal-2021-introducing}), 
for which we also include DynaHate as an OOD test set~\citep{kiela-etal-2021-dynabench},
WinoGrande (multiple choice commonsense reasoning, ~\citealp{sakaguchi2019winogrande}), 
and DialogSum (abstractive summarization, ~\citealp{chen-etal-2021-dialogsum}). 
We use DeBERTaV3 base and large~\citep{he2023debertav} for CAD and WinoGrande,
and OPT 125M and 350M~\citep{zhang2022opt} for DialogSum.\footnote{
    We use relatively small models to avoid huge computation 
    during both training (we trained 1200+ models for controlled comparisons) 
    and gradient projection 
    (which can take $>10$ times longer than training due to high dimensionality). 
} 
We show the ratios of mutually selected instances between 
different representation-selection combinations in Figure~\ref{fig:selection_nlp}. 
Note that because we select 30\% of the data points,
random selection would result in an overlap ratio of 0.3.\footnote{
    Let $N$ be the total size of the dataset. 
    The expected number of overlap items is 
    $|\mathcal{S}_1 \cap \mathcal{S}_2|=0.3N \times 0.3N = 0.09N$. 
    Since $|\mathcal{S}_1|=0.3N$, the overlap ratio is $ 0.09N/0.3N = 0.3$.
}
We also include a discussion of the out-of-distribution settings 
in Appendix~\ref{app_sec:ood_settings}.

\paragraph{Varying representations drastically changes selection}

For example, $h_{\mathcal{M}}$-{difficulty\(_{\text{proto}}\)} and 
$\ell_{\mathcal{M}}$-{difficulty\(_{\text{proto}}\)} 
on synthetic data respectively select instances far from and near the decision boundary
(Figure~\ref{fig:selection_synthetic}), 
and $h_{\mathcal{M}}$-{relevance\(_{\text{less}}\)} and $\nabla_w \ell_{\mathcal{M}}$-{relevance\(_{\text{less}}\)} 
have lower-than-random overlap on both NLP tasks (Figure~\ref{fig:selection_nlp}).
Nevertheless, we find 
the the sensitivity of selection algorithms towards representations varies. 
Particularly, \emph{
    diversity-preserving algorithms are less affected by the representation choice}.
For instance, compared to prototypicality using different representations, 
diversity shows smaller variations in Figure~\ref{fig:selection_synthetic}, 
and similar results are observed in the NLP tasks in Figure~\ref{fig:selection_nlp}.
This is consistent with that 
diversity-preserving algorithms sample evenly from different regions. 

\paragraph{Representations have a larger influence on instance selections 
than the selection algorithms themselves} 
For example, 
$\nabla_w \ell_{\mathcal{M}}$-{difficulty\(_{\text{proto}}\)} 
overlaps more on CAD 
with $\nabla_w \ell_{\mathcal{M}}$-{diversity\(_{\text{s2l}}\)} (0.46) and $\nabla_w \ell_{\mathcal{M}}$-{relevance\(_{\text{less}}\)} (0.79), 
than with $h_{\mathcal{M}}$-{difficulty\(_{\text{proto}}\)} (0.26), see Figure~\ref{fig:selection_nlp}. 
Additionally, 
the selections based on gradients and losses have larger overlap with each other 
than with those based on hidden states.
For example, the overlap when using gradients and losses with the same selection algorithm 
($\nabla_w \ell_{\mathcal{M}}$-{difficulty\(_{\text{proto}}\)} and $\ell_{\mathcal{M}}$-{difficulty\(_{\text{proto}}\)}) 
is as large as 0.75 on CAD. 
The observations here are consistent with our theoretical analysis in \S\ref{subsec:rep_reveal}: 
losses and gradients are more informative.

\paragraph{Selections do not always following their objectives}

Because selection algorithms are typically heuristic-driven to achieve specific objectives, 
it is crucial to assess whether they indeed follow these objectives. 
Surprisingly, our results suggest otherwise: 
(1) On synthetic data, most selections from $h_{\mathcal{M}}$-{difficulty\(_{\text{proto}}\)}---which 
aims to select difficult instances---are 
those that are correctly predicted and far from the decision boundary, which can be considered to be the easier ones. 
(2) Comparing the selections under the same objective, 
we observe that \emph{these selections can be vastly different}. 
For example, even though they both aim to select difficult instances, 
$h_{\mathcal{M}}$-{difficulty\(_{\text{proto}}\)} and $\ell_{\mathcal M}$-{difficulty\(_{\text{htl}}\)} show 
very low consistency in instance selections: 
this divergence is evident in synthetic data, 
where only a few data points are mutually selected (Figure~\ref{fig:selection_synthetic}); 
and in NLP tasks, their overlap ratios are close to the random-guess baseline (Figure~\ref{fig:selection_nlp}). 
This result highlights the need for carefully assessing the consistency 
between selection algorithms and their intended objectives, in future studies.

\begin{figure*}[t]
    \centering
    \begin{subfigure}[t]{0.30\textwidth}
      \centering
      \includegraphics[width=\linewidth]{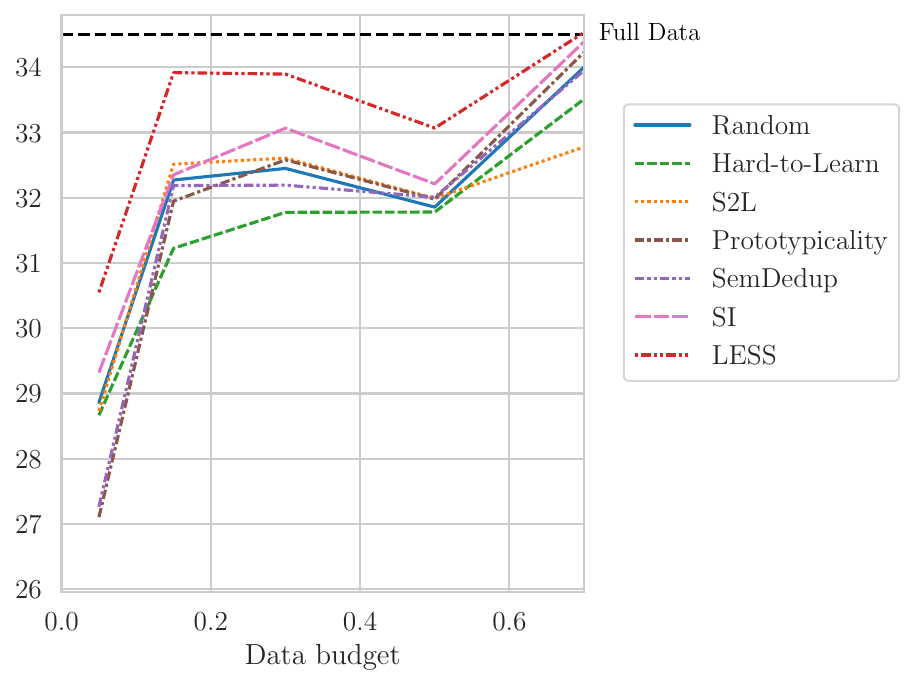}
      \caption{Rouge-L on DialogSum}
      \label{fig:dialogsum-opt-350m}
    \end{subfigure}\hfill
    \begin{subfigure}[t]{0.30\textwidth}
      \centering
      \includegraphics[width=\linewidth]{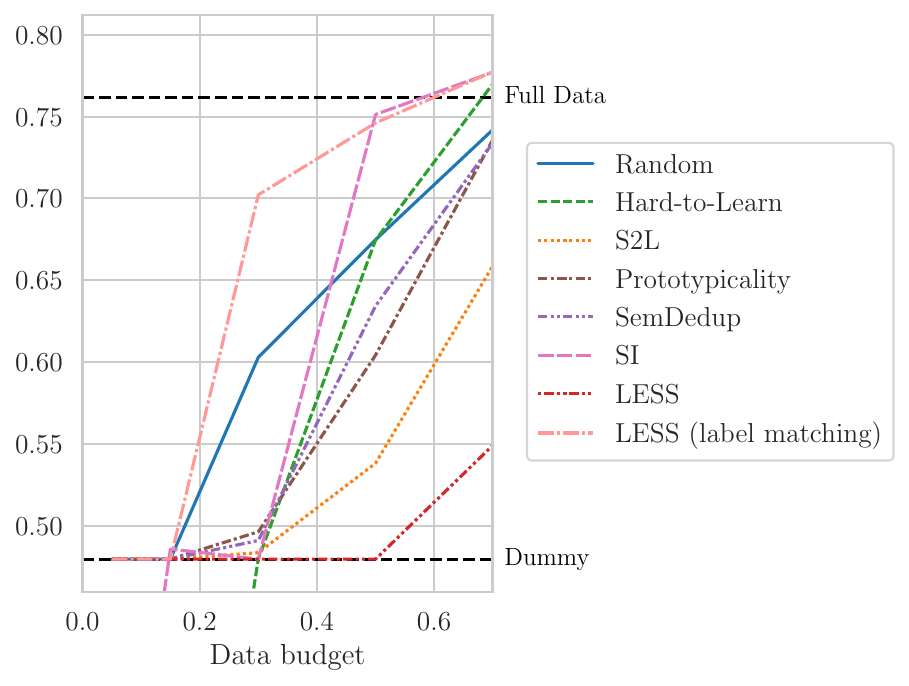}
      \caption{F1 on CAD}
      \label{fig:cad-debertav3-large}
    \end{subfigure}\hfill
    \begin{subfigure}[t]{0.30\textwidth}
      \centering
      \includegraphics[width=\linewidth]{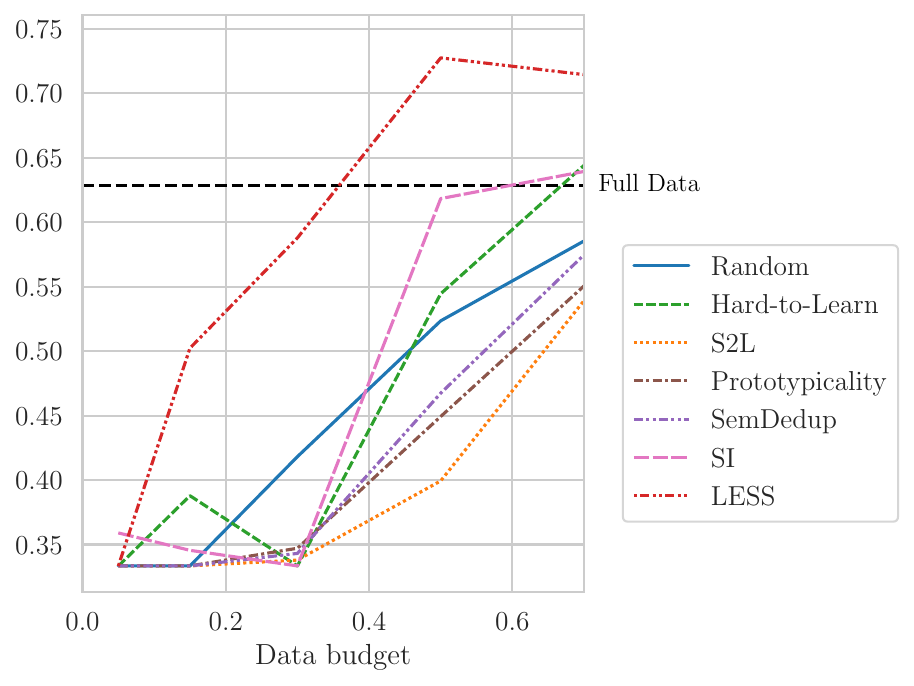}
      \caption{F1 on DynaHate}
      \label{fig:dynahate-debertav3-large}
    \end{subfigure}
  
    \vspace{1em} %
  
    \begin{subfigure}[t]{0.30\textwidth}
      \centering
      \includegraphics[width=\linewidth]{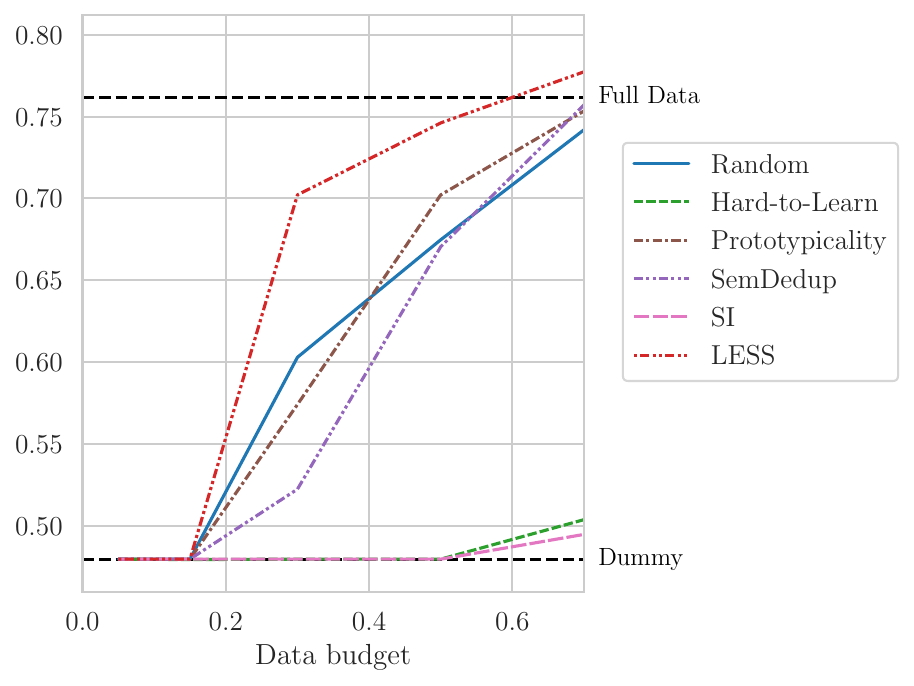}
      \caption{F1 on CAD (label match)}
      \label{fig:cad-debertav3-large-balanced}
    \end{subfigure}\hfill
    \begin{subfigure}[t]{0.30\textwidth}
      \centering
      \includegraphics[width=\linewidth]{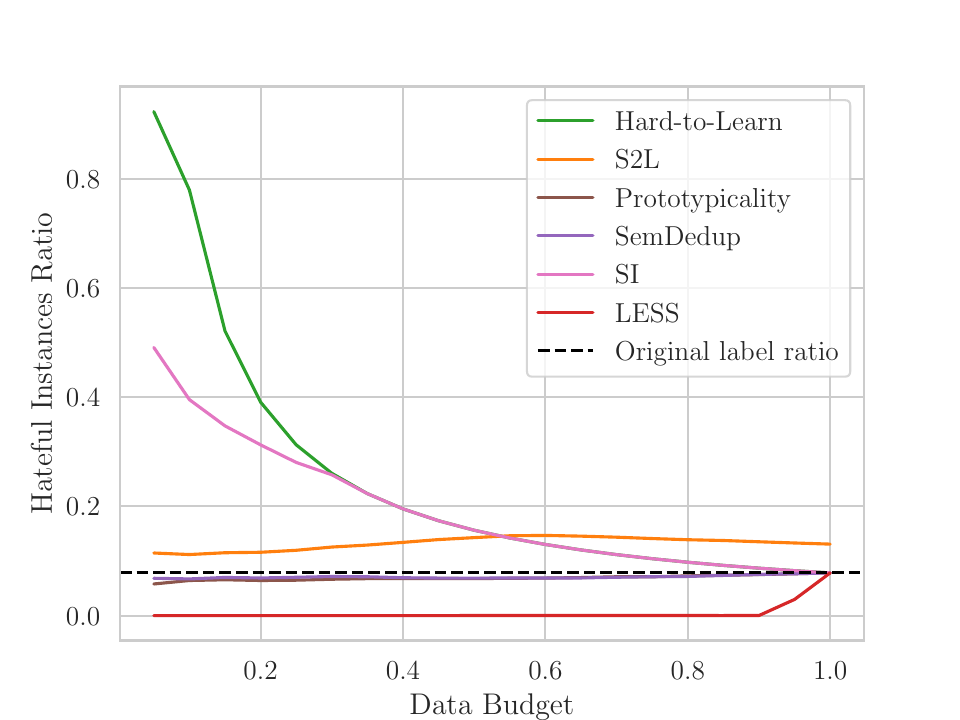}
      \caption{Label distribution on CAD}
      \label{fig:less-cad-balancing}
    \end{subfigure}\hfill
    \begin{subfigure}[t]{0.30\textwidth}
      \centering
      \includegraphics[width=\linewidth]{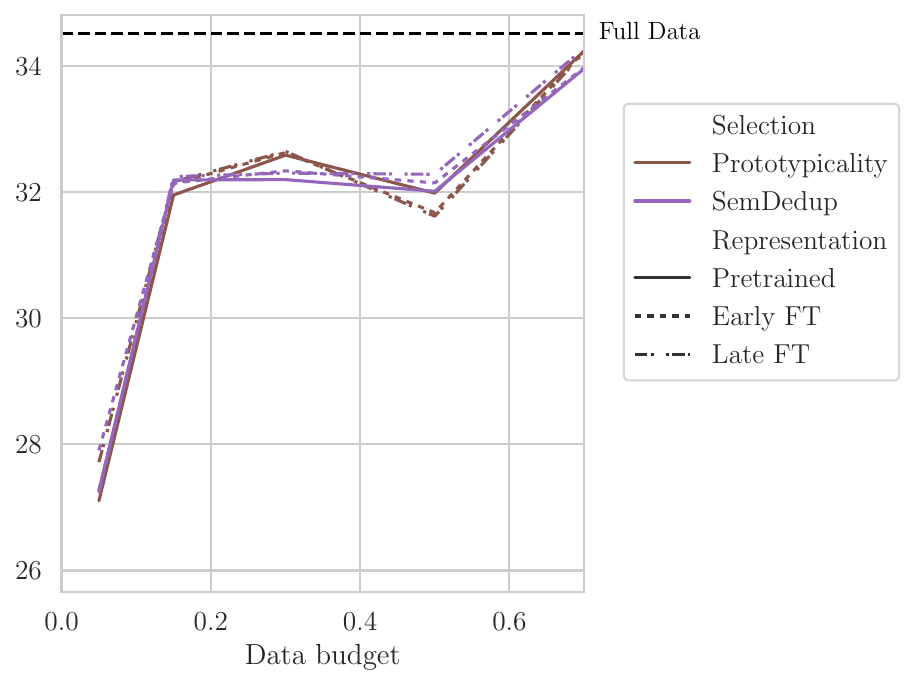}
      \caption{Fine-tuning: DialogSum}
      \label{fig:merged-ft-hiddenstate-dialogsum-opt-350m}
    \end{subfigure}

    \vspace{1em} %
  
    \begin{subfigure}[t]{0.30\textwidth}
      \centering
      \includegraphics[width=\linewidth]{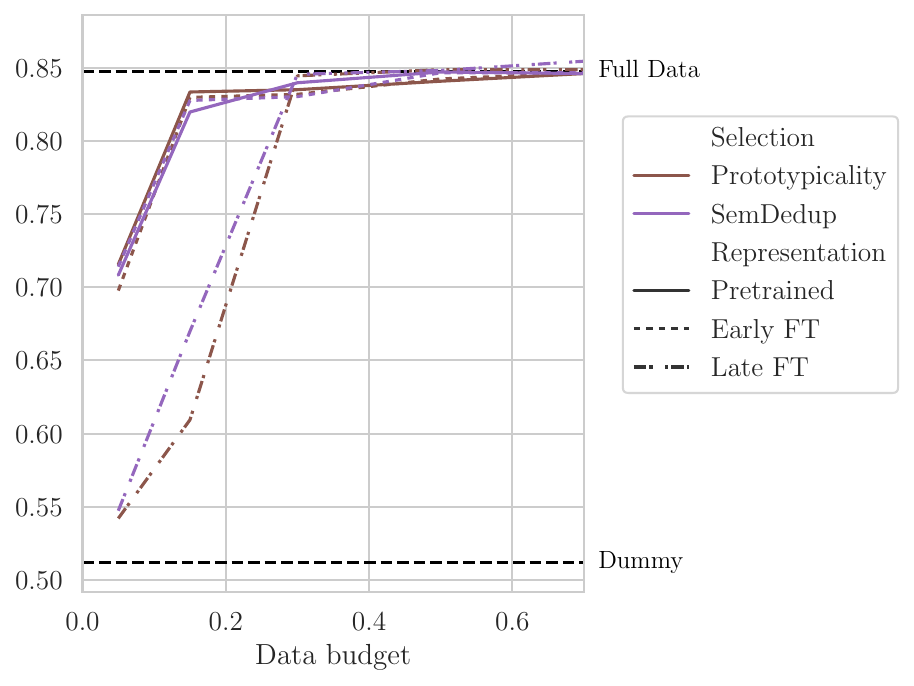}
      \caption{Fine-tuning: WinoGrande}
      \label{fig:merged-ft-hiddenstate-winogrande-debertav3-large}
    \end{subfigure}\hfill
    \begin{subfigure}[t]{0.30\textwidth}
      \centering
      \includegraphics[width=\linewidth]{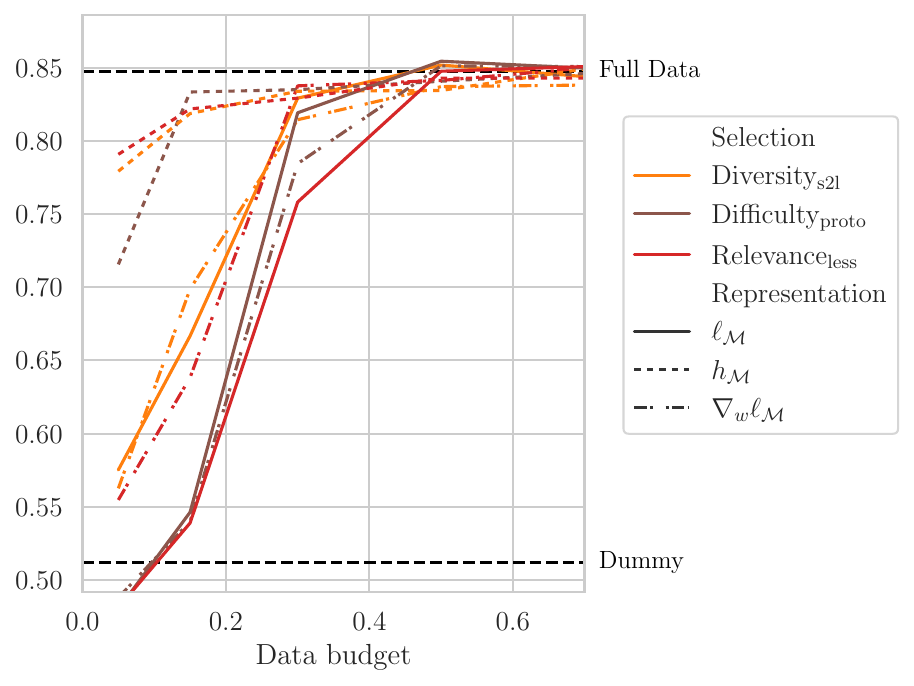}
      \caption{Varying rep.: WinoGrande}
      \label{fig:merged-sampling-feature-winogrande}
    \end{subfigure}\hfill
    \begin{subfigure}[t]{0.30\textwidth}
      \centering
      \includegraphics[width=\linewidth]{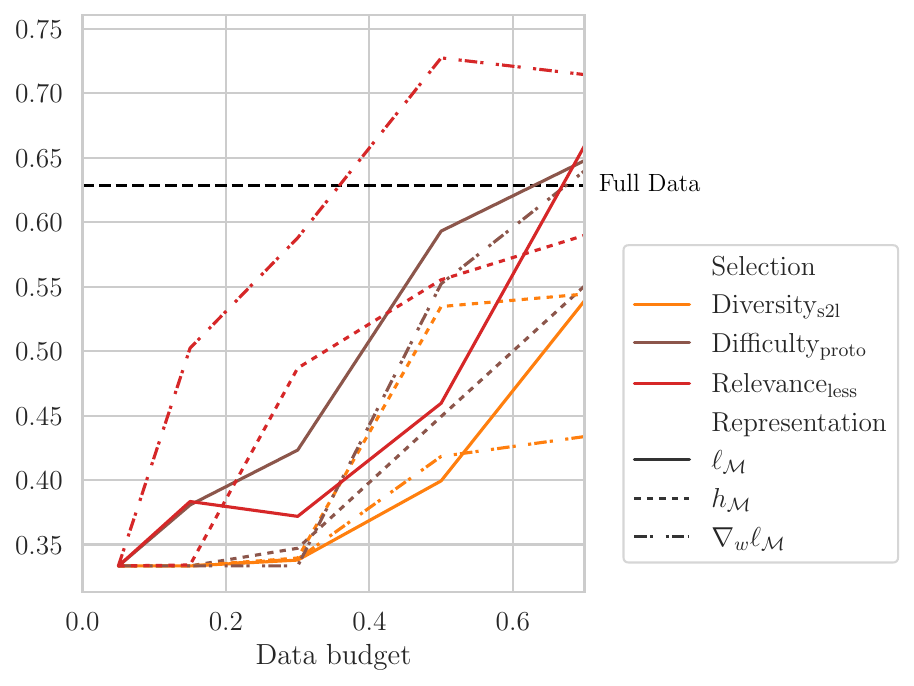}
      \caption{Varying rep.: DynaHate}
      \label{fig:merged-sampling-feature-dynahate}
    \end{subfigure}
  
    \caption{%
      Results with $\text{DeBERTaV3}_{\text{Large}}$ and OPT-350M models.
      (a)–(d) show performance across data budgets,
      (e) presents label distributions,
      while (f) and (g) compare pretrained vs.\ fine-tuned hidden states,
      and (h) and (i) examine representation variation for WinoGrande and DynaHate.
    }
    \label{fig:different_data_budgets}
  \end{figure*}

\section{Performance Across Data Budgets}\label{sec:empirical_study}

Building on our previous analyses on how different components affect data selection 
(\S\ref{subsec:properties_selection}),
this section examines their impact on model performance under different data budgets. 
These experiments help validate our previous findings 
on the effectiveness of different representations and selection algorithms, 
while addressing practical questions about   
\emph{which data pruning methods are best suited for specific tasks and data budgets}, 
such as when handling substantial distribution shift between training and testing data. 
Moreover, we perform two sets of ablation experiments: 
using fine-tuned instead of pretrained hidden states, 
as they may encode task-specific information; 
and experimenting with different representation-selection pairs, 
to better understand the contribution of each component. 

\paragraph{Setup}

Our experiments on NLP tasks follow the same setup as in \S\ref{subsec:properties_selection}.
Moreover, we use six different data budgets
5\%, 15\%, 30\%, 50\%, and 70\% of the original dataset, 
and train all models for 15 epochs. 
Additionally, we consider three baselines: 
random selection (Rand),  the full original dataset (Full Data), and 
a dummy predictor (Dummy), 
which represents the better performance between a randomized predictor and a majority class predictor. 
See Appendix~\ref{app_sec:exp_details} for more details. 

\subsection{Main observations}
We make two main observations from our results 
(see Figure~\ref{fig:different_data_budgets} for representative examples, 
with additional results in Appendix~\ref{app_sec:additional_results}). 
First, selecting the appropriate data pruning method for each specific setting is crucial: 
when they are applied outside their original context, 
\emph{they are often outperformed by random selection}, which is consistent with \citet{okanovic2024repeated}. 
Notably, 
hidden-state-based methods perform worse than or similarly to random selection on all tasks, 
especially with lower data budgets. 
This aligns with our previous results (\S\ref{sec:deciphering}),
that pretrained hidden states may not have sufficient discriminative power 
to select the most important instances for the model parameters.\footnote{
    Note that our experiments differ from previous studies that use hidden states,  
    as they only experimented under high data budget settings~\citep{sorscher2022beyond} 
    and noisy pretraining datasets~\citep{abbas2023semdedup}. 
}
Similarly, higher data budgets are needed for other methods 
that aim to select difficult instances, i.e., hard-to-learn and self-influence: 
they achieve competitive performance with $>30\%$ data budgets, 
but are only comparable to the dummy baseline with lower data budgets.
This is consistent with the observations in \citet{swayamdipta-etal-2020-dataset}, 
that using only the hardest instances will make models fail to converge. 

Second, gradient-based methods like LESS and self-influence perform competitively 
across most tasks (Figure~\ref{fig:dialogsum-opt-350m} \& \ref{fig:dynahate-debertav3-large}), 
\emph{reaffirming the effectiveness of gradients as data representations}. 
Interestingly, LESS performs well on highly imbalanced datasets 
only when applying \textbf{label matching}, 
i.e., enforcing selected instances to maintain the original label ratio. 
Without this constraint, 
LESS tends to over-select majority-label training instances, 
because it selects training instances based on their distance to validation data. 
Since instances with the same label tend to have shorter distances 
between their gradient representations (\S\ref{subsec:rep_reveal}), 
LESS would further amplify this bias. 
We validate this on CAD, 
where only 10\% of the instances are hateful (Figure~\ref{fig:less-cad-balancing}), 
by plotting the ratio of hateful labels in selected instances. 
Figure~\ref{fig:less-cad-balancing} shows that 
LESS primarily selects non-hateful instances, with very few hateful ones included 
until the data budget reaches 90\%.
However, this constraint can be detrimental to methods like hard-to-learn, 
which prioritizes rare instances such as hateful ones.

\subsection{Ablation analysis}

\paragraph{Fine-tuned hidden states}
We previously observed that 
hidden-state-based methods perform similarly to random selection.
Since effective representations should encode 
label information and the distance to the decision boundary, 
we examine whether fine-tuning can help improve their performance. 
Specifically, we compare two types of hidden states when models stop training at 
early (one epoch, retaining more pretrained knowledge) 
and late (15 epochs, encoding more task-specific information) stages, 
as shown in Figure~\ref{fig:merged-ft-hiddenstate-dialogsum-opt-350m} and \ref{fig:merged-ft-hiddenstate-winogrande-debertav3-large}. 
However, there is little difference between using different hidden states, 
with none of them clearly outperforming random selection, 
suggesting the insufficiency of fine-tuning. 

\paragraph{Varying representations}

We have shown that data representations have a greater influence 
on selections than the selection algorithms themselves (\S\ref{sec:deciphering}).
To study whether this holds for model performance,
we similarly use different representations with 
diversity\(_{\text{s2l}}\), difficulty\(_{\text{proto}}\), and relevance\(_{\text{less}}\),
and show the results for $\text{DeBERTaV3}_{\text{Large}}$ on WinoGrande and DynaHate 
in Figures~\ref{fig:merged-sampling-feature-winogrande} \& \ref{fig:merged-sampling-feature-dynahate}.
In line with \S\ref{subsec:properties_selection},
we find that \emph{similar representations often lead to similar performance} 
(Figure~\ref{fig:merged-sampling-feature-winogrande}).
Nevertheless, the best selection algorithm is still task-dependent.
For example, 
when there is a train-test distribution mismatch
(e.g., DynaHate in Figure~\ref{fig:merged-sampling-feature-dynahate}),
LESS performs better than other methods by considering validation performance.

\section{Conclusion}\label{sec:conclusion}
Despite the success of data pruning, 
the contributions of its design choices have remained unclear. 
This paper has identified two key components: data representations and selection algorithms, 
and provided a comprehensive overview of common choices (\S\ref{sec:decoupling}).
Moreover, we have provided both 
theoretical and controlled empirical analyses on their effectiveness (\S\ref{sec:deciphering}), 
and their implications across different data budgets (\S\ref{sec:empirical_study}).
Our results highlight the critical role of data representations 
due to their impact on selected instances, 
and the importance of evaluating selection algorithms carefully,
as they are not guaranteed to meet their objectives. 
Our findings stress the need for 
the development of efficient and informative data representations. 

\section*{Limitations}\label{sec:limitations}

One limitation of our work is its focus on task-specific fine-tuning, 
leaving other settings, 
like pretraining, supervised fine-tuning, and reinforcement learning from verifiable rewards, unexplored. 
This is largely due to 
(1) the large amount of computation required to conduct rigorous controlled studies such as ours, 
and (2) the challenges in scalable and low-cost evaluation~\citep{zheng2023judging}. 
Future studies could explore data pruning approaches in these settings 
by generating synthetic training and validation tasks, 
which allows for low-cost and controlled studies. 
This has been recently shown to be useful 
for proof-of-concept studies~\citep{allenzhu2024physicslanguagemodels31}, 
whose conclusions generalize to larger scale settings well. 
Moreover, we focus on methods that do not require external models 
(e.g., prompting language models to evaluate example quality).  
Future work could expand our analyses to include such approaches.

\section*{Acknowledgments}
We thank the anonymous reviewers for their helpful comments. 
We thank members of the NLP group at Utrecht University for their feedback, 
especially to Anna Wegmann. 
We thank the members of the MaiNLP group at LMU Munich for their feedback, 
especially to Barbara Plank and Philipp Mondorf.  
This work is supported by the ERC Starting Grant DataDivers 101162980.

\bibliography{references}

\clearpage
\appendix

\section{Experimental Details}\label{app_sec:exp_details}

\paragraph{Implementation Details}

All experiments were conducted using the AdamW optimizer. 
For most models, we set the learning rate to 2e-5, 
except for $\text{DeBERTaV3}_{\text{Large}}$, 
where we followed \citet{he2023debertav} and used 1e-5. 
Additionally, we do a learning rate warmup for the first 10\% of training steps. 
For gradient-based pruning, 
the reference models were trained with LoRA, 
using a higher learning rate of 1e-4, $r=64$, and $\alpha=16$
following~\citet{ivison2023camels}, 
and apply LoRA on all linear layers.
We train all models for 15 epochs,
For batch size, 
we used 16 for both WinoGrande and DialogSum, and 32 for CAD, 
to fit all experiments on a single NVIDIA A100-40GB GPU.
We use maximum sequence lengths of 300, 128, and 512 tokens. 
For all experiments we use the same reference models as the main models for fair comparison.

For $k$-Means clustering in S2L, Prototypicality, and SemDedup,
we use 100 clusters on CAD and DialogSum, and 200 clusters on WinoGrande, 
following the suggestions from~\citet{NEURIPS2023_a8f8cbd7} 
to set the number of clusters to around the square root of the number of instances.
Moreover, 
we compute gradients using the first five checkpoints for all experiments,
and project them into a 1024-dimensional space using \citet{park2023trak} 
(details see hyperparameter search).

\paragraph{Evaluation Metrics}

We evaluated CAD and DynaHate using the macro F1 score, 
WinoGrande by accuracy, 
and DialogSum by ROUGE-1, ROUGE-2, and ROUGE-L 
(from HuggingFace Evaluate), 
following the original studies~\citep{
    vidgen-etal-2021-introducing,sakaguchi2019winogrande,chen-etal-2021-dialogsum}.

\paragraph{Infrastructure}

All experiments were run on a single NVIDIA A100-40GB GPU using three random seeds. 
We used PyTorch 2.3, Transformers 4.42, and vLLM 0.5 for training and inference. 
Moreover, we use {\tt bfloat16} on all experiments to improve efficiency. 

\paragraph{Hyperparameter Search}

We searched for four hyperparameters:
the number of training epochs, 
the number of clusters for $k$-Means clustering,
the dimensionality of the projected gradients,
and the checkpoints to use for gradient computation.

For the number of training epochs, 
we first perform a search over 3, 5, 7, and 10 epochs on all datasets and models,
using three random seeds. 
We observe that models of different sizes share similar performance trends over epochs, 
with improvements continuing as the number of epochs increased. 
We therefore use the smaller models, 
i.e., $\text{DeBERTaV3}_{\text{Base}}$ and OPT-125M, 
and extend this search over 15, 20, and 25 epochs. 
Across all datasets, the best performance is achieved with 15 epochs.

For the number of clusters,
we search over 2, 5, 10, 20, 50, 100, and 200 clusters 
for each dataset and model, using three random seeds.
The results are highly consistent across cluster numbers.
Following \citet{NEURIPS2023_a8f8cbd7}, 
we use the square root of the dataset size as a guideline, 
settling on 100 clusters for CAD and DialogSum, and 200 for WinoGrande.

For gradients, 
we use smaller models ($\text{DeBERTaV3}_{\text{Base}}$ and OPT-125M) for hyper-parameter search, 
and only one random seed (0) to avoid the high costs of computing and projecting gradients. 
We compute the gradients for all 15 checkpoints,
and project them into 1024, 2048, and 4096 dimensions.
First, we observe that different projections yield similar results,
and thus choose 1024 for further experiments for efficiency.
Second, we experimented with different strategies for selecting checkpoints,
including 
the first three, the last three, 
the first five, the last five, 
and evenly spaced three and five checkpoints. 
Using the first checkpoints is the most consistent with using all checkpoints, 
with the first five yielding a minimum Spearman's rank correlation of 0.96. 
We therefore use the first five checkpoints for all experiments.

\section{Overview of Data Pruning Methods}\label{app_sec:overview_dpm}

\paragraph{Hard-to-Learn (training dynamics)}

The Hard-to-Learn method is based on a simple intuition: 
training instances that are \textbf{difficult} for models to fit often contain fewer regular patterns 
and can thus improve model generalization~\citep{
    swayamdipta-etal-2020-dataset,pmlr-v139-jiang21k}.
In classification tasks, the score of an instance $(x_i,y_i)$ is defined as 
\emph{the average prediction probability of the correct label across different epochs}, 
i.e., $\frac{1}{T} \sum_{t=1}^{T} p_{\mathcal{M}}^{(t)}(y_i|x_i)$. 
The main model is then trained on instances with the lowest scores. 
Originally proposed for classification tasks, 
\citet{bhatnagar-etal-2022-chia} and \citet{ince2023harnessing} 
extend this concept to generation tasks, 
by replacing the minus average prediction probability with the inverse perplexity. 

\paragraph{SmallToLarge (training dynamics)}

SmallToLarge (S2L; \citealp{Yang2024SmallToLargeS}) is proposed 
to select \textbf{diverse} instances, 
to preserve full-dataset knowledge during supervised fine-tuning. 
Noting that similar loss trajectories indicate similar knowledge, 
S2L performs three steps to ensure the diversity of the selected data. 
First, each training instance is represented by its \emph{cross entropy loss trajectory} 
observed during reference model training. 
S2L then performs $k$-means clustering on these trajectories. 
Finally, S2L iteratively samples from each cluster, 
while balancing the number of instances across clusters.

\paragraph{Prototypicality (hidden states)}

The Prototypicality method \citep{sorscher2022beyond} selects \textbf{difficult} instances in pretraining, 
by exploiting their \emph{similarities}: 
it measures difficulty based on how \emph{prototypical} an instance is.
Specifically, after representing instances by their \textbf{hidden states}, 
prototypicality applies $k$-means clustering and 
ranks instances based on their distances to their cluster centroids. 
Instances with larger distances 
are considered less prototypical and therefore more difficult, 
and are thus selected to train the main model.

\paragraph{SemDeDup (hidden states)}
Building on Prototypicality, targeting large-scale pretraining, 
SemDeDup includes an additional step to also account for data diversity~\citep{abbas2023semdedup}:
after clustering, 
it identifies semantically duplicate pairs of instances within each cluster 
using cosine similarities of their \textbf{hidden states}. 
For each identified duplicate pair, 
it retains the instance that lies farther from the cluster centroid, 
thereby prioritizing \textbf{diversity} while maintaining \textbf{difficulty}.

\paragraph{LESS (gradients)}

Proposed for supervised fine-tuning, 
LESS additionally requires a validation set to select more \textbf{relevant} instances, 
using cosine similarities between \textbf{gradients}~\citep{xia2024less}. 
Formally, the relevance of $(x_i, y_i)$ 
w.r.t. a validation instance $(x_\mathrm{val}, y_\mathrm{val})$ is defined as 
$\sum_{t=1}^{T} \eta_t \cdot 
\nabla_\theta \ell_{\mathcal{M}}^{(t)}(x_i, y_i)^{\top}
\nabla_\theta \ell_{\mathcal{M}}^{(t)}(x_\mathrm{val}, y_\mathrm{val}))$, 
where $\eta_t$ is the average learning rate between the $t$-th and the $t+1$-th checkpoint. 
These gradients are normalized in generation tasks 
because their norms negatively correlate with sequence lengths. 

\paragraph{Self-Influence (gradients)}

\citet{NEURIPS2020_1e14bfe2} define memorization during training  
as the prediction probability decrease of an instance 
before and after removing it from the training set, i.e., self-influence. 
They argue that memorized instances are usually \textbf{difficult}-to-predict, 
and thus contribute more to generalization 
under the long-tail assumption of testing cases~\citep{10.1145/3357713.3384290}. 
In this work, following \citet{bejan-etal-2023-make}, 
we use TracIn~\citep{NEURIPS2020_e6385d39} for approximation. 
Formally, the self-influence score of $(x_i, y_i)$ is estimated as 
$\sum_{t=1}^{T} \eta_t \nabla_{\theta} \ell_{\mathcal{M}}^{(t)}(x_i, y_i)^{\top}
\nabla_{\theta} \ell_{\mathcal{M}}^{(t)}(x_i, y_i)$.

\section{Discussion of out-of-distribution settings}\label{app_sec:ood_settings}
We focus on in-distribution settings in this paper, 
where the validation and test data come from the same distribution, 
although the train data may come from a different distribution, 
e.g., CAD(train)-DynaHate(validation/test) in HSD. 
Here we discuss the potential impact of the out-of-distribution (OOD) settings, 
where the training and validation data come from different distributions.

Regarding representations, we expect they perform similarly as in in-distribution settings. 
For example, gradients should be better: 
their discriminative power can help model build a more robust decision boundary; 
meanwhile, we still expect hidden states to be less effective,
because they lack the discriminative power to distinguish instances 
that are different for model training. 

Regarding selection objectives, 
we expect methods that prioritize difficult instances to perform better, 
because easy training instances usually contain more regularities and shortcuts. 
For example, in NLI, contradiction with negation words 
usually are considered ``easier'' than the ones without negation words, 
because they are seen more frequently in the training data; 
meanwhile, we expect methods that prioritize relevance to perform worse, 
because this might drive the training data distribution 
further away from the test data distribution.

\clearpage

\section{Proofs}\label{app_sec:proofs}
\subsection{Derivation of Remark~\ref{remark:explicit}}\label{app_sec:proof_explicit}
We first restate the remark for reference.

\begin{remark}[Explicit expressions]
    Let $z_{*} = y_{*} w^T h_{\mathcal{M}}(x_{*})$ be the (signed and scaled) distance 
    from $h_{\mathcal{M}}(x_{*})$ to the decision boundary.
    We have
    $D_{\ell} = (\log{((1 + e^{-z_i}) / (1 + e^{-z_j})}))^2$, 
    $
    D_g = \|\frac{y_i h_{\mathcal{M}}(x_i)}{1+e^{z_i}} - \frac{y_j h_{\mathcal{M}}(x_j)}{1+e^{z_j}}\|_2^2
    = \frac{\|h_{\mathcal{M}}(x_i)\|_2^2}{(1+e^{z_i})^2} + \frac{\|h_{\mathcal{M}}(x_j)\|_2^2}{(1+e^{z_j})^2}
    - 2 \frac{y_i y_j h_{\mathcal{M}}(x_i)^T h_{\mathcal{M}}(x_j)}{(1+e^{z_i})(1+e^{z_j})}$.
\end{remark}

\begin{proof}[Derivation]
    We first derive the expression for $D_{\ell}$. 
    \begin{IEEEeqnarray}{rll}
        D_{\ell} 
        &:= \|\ell_{\mathcal{M}}(x_i, y_i) - \ell_{\mathcal{M}}(x_j, y_j)\|_2^2 \\
        &= \|\log(1 + e^{-y_i w^T h_{\mathcal{M}}(x_i)}) \nonumber \\
        &\hspace{1.5cm}- \log(1 + e^{-y_j w^T h_{\mathcal{M}}(x_j)})\|_2^2 \\
        &= (\log(1 + e^{-z_i}) - \log(1 + e^{-z_j}))^2 \\
        &= (\log(\frac{1 + e^{-z_i}}{1 + e^{-z_j}}))^2
    \end{IEEEeqnarray}

    Next, we derive the expression for $D_g$. Recall that in \S\ref{sec:deciphering}, we defined $D_g := \|\nabla_w \ell_{\mathcal{M}}(x_i, y_i) - \nabla_w \ell_{\mathcal{M}}(x_j, y_j)\|_2^2$. We first derive $\nabla_w \ell_{\mathcal{M}}(x_i, y_i) = \frac{y_i h_{\mathcal{M}}(x_i)}{1+e^{z_i}}$.

    \begin{IEEEeqnarray}{rll}
        \nabla_w \ell_{\mathcal{M}}(x_i, y_i)
        &= \nabla_w \log(1 + e^{-y_i w^T h_{\mathcal{M}}(x_i)}) \\
        &= \nabla_w \log(1 + e^{-z_i}) \\
        &= -(1 - \sigma(z_i))y_i h_{\mathcal{M}}(x_i) \\
        &= -\frac{y_i h_{\mathcal{M}}(x_i)}{1 + e^{z_i}}
    \end{IEEEeqnarray}

    The derivation of $\nabla_w \ell_{\mathcal{M}}(x_j, y_j) = \frac{y_j h_{\mathcal{M}}(x_j)}{1+e^{z_j}}$ is analogous to the one above, and thus we get $D_g = \|\frac{y_i h_{\mathcal{M}}(x_i)}{1+e^{z_i}} - \frac{y_j h_{\mathcal{M}}(x_j)}{1+e^{z_j}}\|_2^2$.
\end{proof}

\subsection{Proof of Theorem~\ref{theorem:discriminative}}\label{app_sec:proof_discriminative}

We first restate the theorem for reference.

\begin{theorem}
    The discriminative power of losses and gradients (relative to that of hidden states) 
    are dependent of 
    the region the hidden states lie in, satisfying
    $$
    \begin{aligned}
    \mathcal{C}_{\ell} &= \frac{e^{z_i}}{1+e^{z_i}} \|w\|, \text{ and } \\
    \mathcal{C}_g &\leq 
    \frac{1}{1+e^{z_i}}+\frac{e^{z_i}}{\left(1+e^{z_i}\right)^2}\|h_{\mathcal{M}}(x_i)\| \|\|w\|.
    \end{aligned}
    $$
\end{theorem}

\begin{proof}
    We first prove the discriminative power of losses.
    Let $\sigma(z_i) = \frac{1}{1+e^{-z_i}}$, we have
    \begin{IEEEeqnarray}{rll}
        \frac{\partial \ell_{\mathcal{M}}(x_i, y_i)}{\partial h_{\mathcal{M}}(x_i)}
        &= -\frac{\partial \log \sigma(z_i)}{\partial z_i}
        \frac{\partial z_i}{\partial h_{\mathcal{M}}(x_i)} 
        \\
        &= -(1-\sigma(z_i)) y_i w \\
        &= -\frac{1}{1+e^{z_i}} w \\ 
        &= -(1 - p_{\mathcal{M}}(y_i|x_i)) w.
    \end{IEEEeqnarray}
    Therefore, 
    \begin{IEEEeqnarray}{rll}
        \mathcal{C}_{\ell} 
        = \|\frac{\partial \ell_{\mathcal{M}}(x_i, y_i)}{\partial h_{\mathcal{M}}(x_i)}\|
        = \frac{1}{1+e^{z_i}} \|w\|.
    \end{IEEEeqnarray}

    Next, we prove the discriminative power of gradients.
    We have 
    \begin{IEEEeqnarray}{rll}
        &\frac{\partial \nabla_{w} \ell_{\mathcal{M}}(x_i, y_i)}{\partial h_{\mathcal{M}}(x_i)} \\
        = &\frac{\partial ((y_i h_{\mathcal{M}}(x_i))/(1+e^{z_i}))}{\partial h_{\mathcal{M}}(x_i)} \\
        = &\frac{y_i}{1+e^{z_i}} I - \frac{y_i h_{\mathcal{M}}(x_i) e^{z_i}}{(1+e^{z_i})^2} 
        \frac{\partial z}{\partial h_{\mathcal{M}}(x_i)} \\
        = &\frac{y_i}{1+e^{z_i}} I - \frac{e^{z_i}}{(1+e^{z_i})^2} h_{\mathcal{M}}(x_i) w^{\top}.
    \end{IEEEeqnarray}
    Therefore, 
    \begin{IEEEeqnarray}{lll}
        &\mathcal{C}_g 
        = \|\frac{\partial \nabla_{w} \ell_{\mathcal{M}}(x_i, y_i)}{\partial h_{\mathcal{M}}(x_i)}\| \\
        &= \|\frac{1}{1+e^{z_i}} I - \frac{e^{z_i}}{(1+e^{z_i})^2} h_{\mathcal{M}}(x_i) w^{\top}\| \\
        &\leq \| \frac{1}{1+e^{z_i}} I \| + \|\frac{e^{z_i}}{(1+e^{z_i})^2} h_{\mathcal{M}}(x_i) w^{\top}\| \\
        &= \frac{1}{1+e^{z_i}} + \frac{e^{z_i}}{(1+e^{z_i})^2} \| h_{\mathcal{M}}(x_i) \| \|w\|.
    \end{IEEEeqnarray}
\end{proof}

\subsection{Proof of Corollary~\ref{corollary:gradient_peak}}\label{app_sec:proof_gradient_peak}

We first restate the corollary for reference.

\begin{corollary}
    Let $\|w\| \|h_{\mathcal{M}}(x_i)\| = \alpha$.
    When $\alpha$ is smaller than the positive root of function 
    $-x (1-e^{x}) = 1 + e^{x}$ ($\sim 1.544$), 
    $\mathcal{C}_g$ monotonically decreases as $z_i$ increases, 
    which is similar to $\mathcal{C}_{\ell}$.
    However, when $\alpha$ is larger, 
    $\mathcal{C}_g$ increases with $z_i$ when $z_i \leq \log(\frac{\alpha-1}{\alpha+1})$, 
    which is negatively close to the decision boundary,
    and decreases when $z_i > \log(\frac{\alpha-1}{\alpha+1})$.
\end{corollary}

\begin{proof}
    
To study the behavior of $\mathcal{C}_g$ with $z_i$, 
we take the derivative of $\mathcal{C}_g$ with respect to $z_i$,
\begin{IEEEeqnarray}{rll}
    &\frac{\partial \mathcal{C}_g}{\partial z_i} 
    = \frac{\partial}{\partial z_i} (\frac{1}{1+e^{z_i}} + \frac{e^{z_i}}{(1+e^{z_i})^2} \alpha) \\
    &= -\frac{e^{z_i}}{(1+e^{z_i})^2} + \frac{e^{z_i}}{(1+e^{z_i})^2} \alpha - \frac{2e^{2z_i}}{(1+e^{z_i})^3} \alpha \notag \\
    &= \frac{e^{z_i}}{(1+e^{z_i})^2} (-1 + \alpha - \frac{2e^{z_i}}{1+e^{z_i}} \alpha) \notag \\
    &= \frac{e^{z_i}}{(1+e^{z_i})^3} (\alpha(1 - e^{z_i}) - (1+e^{z_i})) \notag \\ 
    &= \frac{e^{z_i}}{(1+e^{z_i})^3} ((\alpha - 1) - (\alpha + 1)e^{z_i}).\label{eq:alpha_cond}
\end{IEEEeqnarray}
For the domain of $z_i$, we know that 
\begin{IEEEeqnarray}{rll}
    z_i = y_i w^{\top} h_{\mathcal{M}}(x_i) = \alpha \cos(\phi) \in [-\alpha, \alpha]
\end{IEEEeqnarray}
, where $\phi$ is the angle between $w$ and $h_{\mathcal{M}}(x_i)$.

When $z_i > 0$, i.e., the prediction is correct,
$\partial \mathcal{C}_g/\partial z_i < 0$. 
Therefore, $\mathcal{C}_g$ monotonically decreases as $z_i$ increases. 

When $z_i < 0$, i.e., the prediction is incorrect: 
\begin{itemize}
    \item If $0 < \alpha \leq 1$, $\alpha - 1 \leq 0$. 
    From Eq.~\ref{eq:alpha_cond}, 
    we know that $\partial \mathcal{C}_g/\partial z_i < 0$. 
    Therefore, $\mathcal{C}_g$ monotonically decreases as $z_i$ increases.
    \item If $\alpha > 1$, when $z_i \leq \log(\frac{\alpha-1}{\alpha+1})$,
    $\partial \mathcal{C}_g/\partial z_i > 0$.
    Therefore, $\mathcal{C}_g$ increases with 
    $z_i \in [-\alpha, \log(\frac{\alpha-1}{\alpha+1}))$ (if exists), 
    then decreases when $z_i \geq \log(\frac{\alpha-1}{\alpha+1})$.
    To make sure the range exists, 
    we need $\log(\frac{\alpha-1}{\alpha+1}) > -\alpha$, 
    which is equivalent to $\alpha > \sim 1.544$.
    Otherwise, similar to the case of $\alpha \leq 1$,
    $\mathcal{C}_g$ monotonically decreases as $z_i$ increases.
\end{itemize}

\end{proof}

\section{Additional Results}\label{app_sec:additional_results}

\begin{figure*}[t]
    \centering
    \begin{subfigure}[b]{0.32\textwidth}
        \centering
        \includegraphics[width=\linewidth]{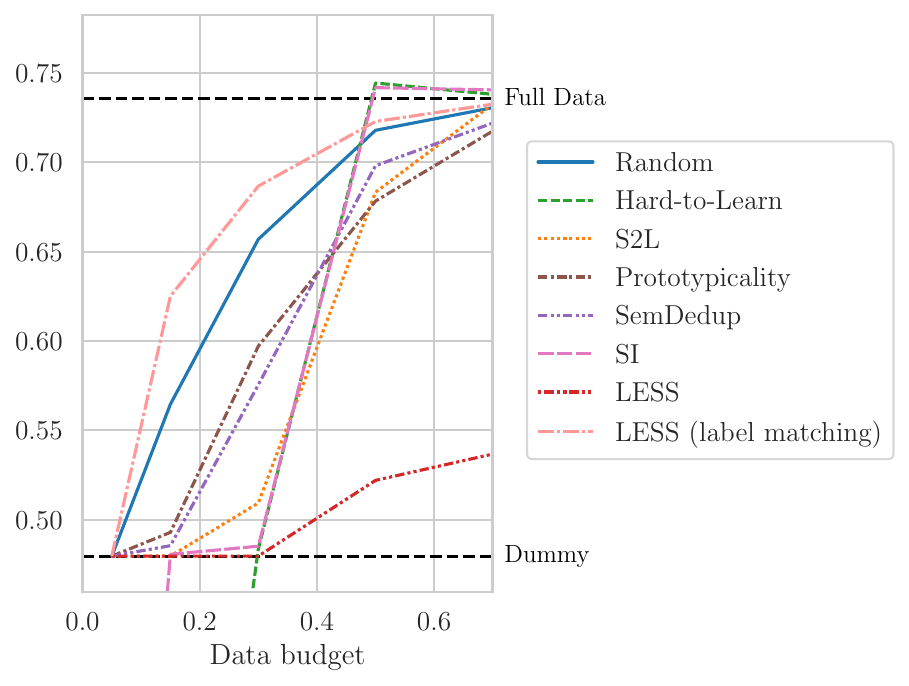}
        \caption{F1 scores of $\text{DeBERTaV3}_{\text{Base}}$ on CAD}
        \label{fig:retrain-cad-debertav3-base}
    \end{subfigure}
    \hfill
    \begin{subfigure}[b]{0.32\textwidth}
        \centering
        \includegraphics[width=\linewidth]{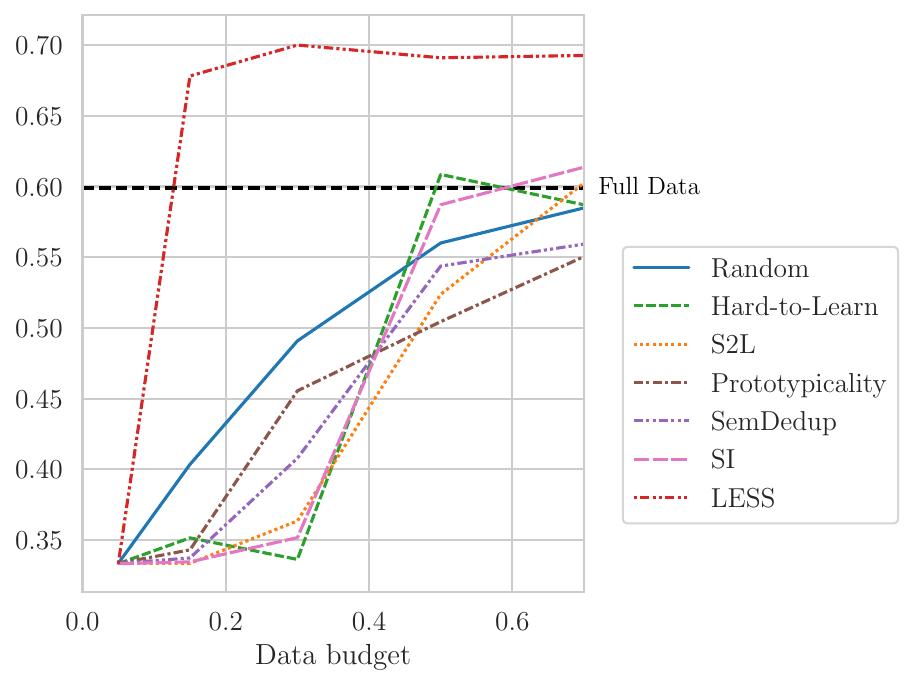}
        \caption{F1 scores of $\text{DeBERTaV3}_{\text{Base}}$ on DynaHate}
        \label{fig:retrain-dynahate-debertav3-base}
    \end{subfigure}
    \hfill
    \begin{subfigure}[b]{0.32\textwidth}
        \centering
        \includegraphics[width=\linewidth]{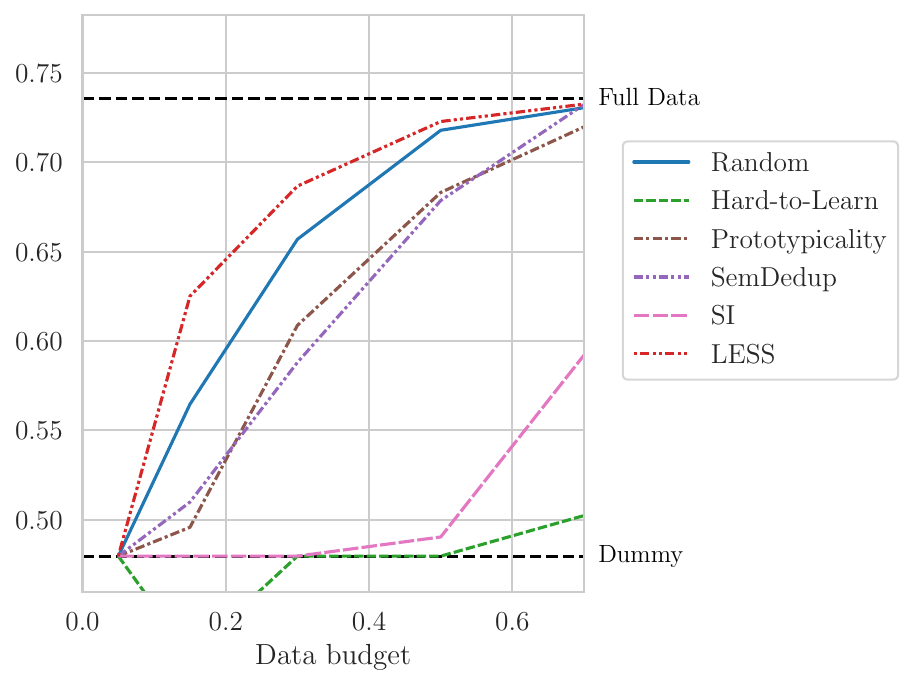}
        \captionsetup{font=small}
        \caption{F1 scores of $\text{DeBERTaV3}_{\text{Base}}$ on CAD with label balancing}
        \label{fig:retrain-cad-debertav3-base-balanced}
    \end{subfigure}
    \par\bigskip
    \begin{subfigure}[b]{0.32\textwidth}
        \centering
        \includegraphics[width=\linewidth]{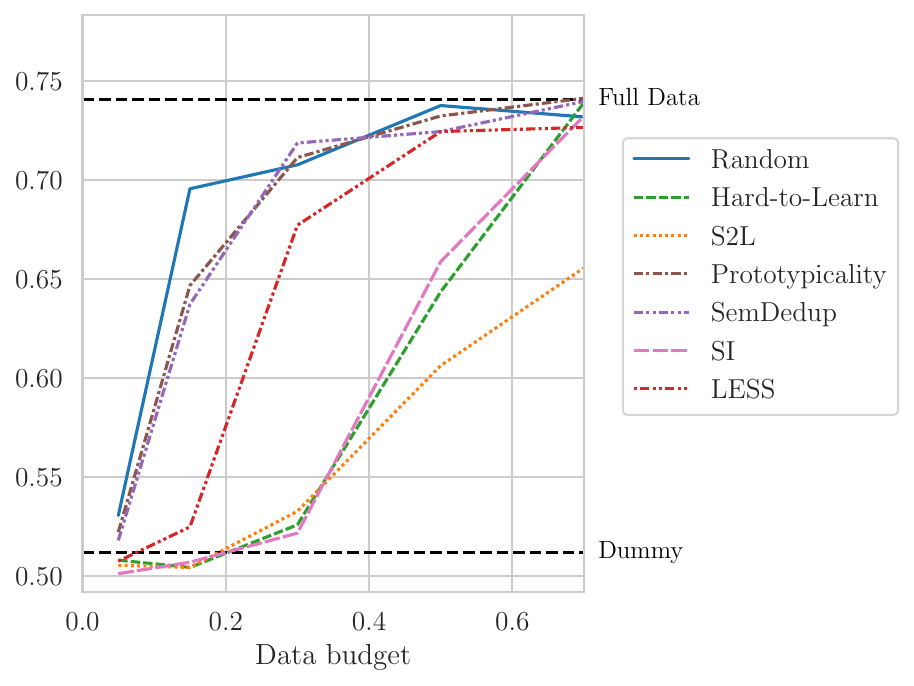}
        \caption{Accuracy of $\text{DeBERTaV3}_{\text{Base}}$ on WinoGrande}
        \label{fig:retrain-winogrande-debertav3-base}
    \end{subfigure}
    \hfill
    \begin{subfigure}[b]{0.32\textwidth}
        \centering
        \includegraphics[width=\linewidth]{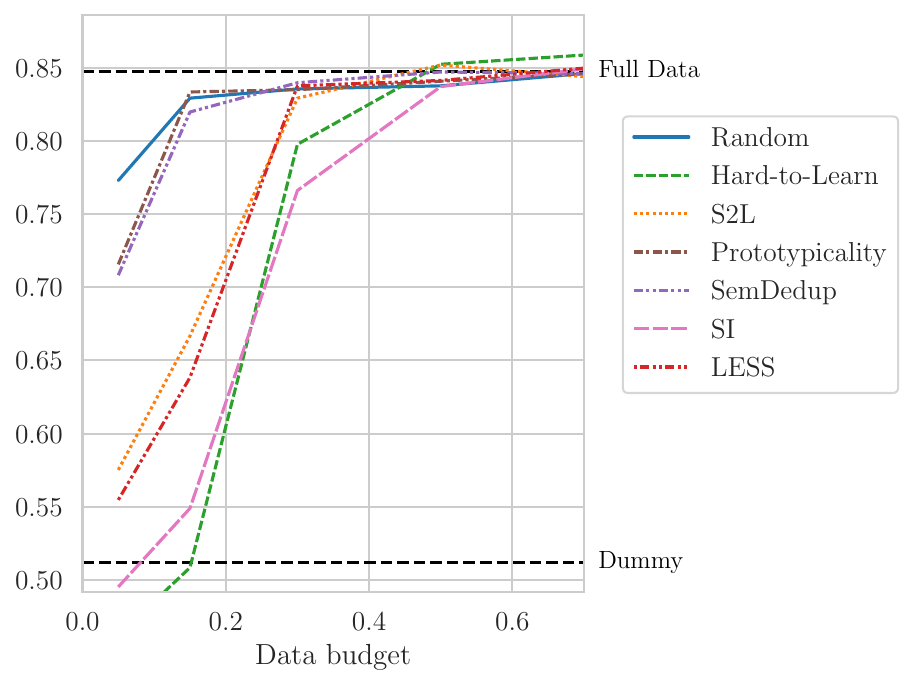}
        \caption{Accuracy of $\text{DeBERTaV3}_{\text{Large}}$ on WinoGrande}
        \label{fig:retrain-winogrande-debertav3-large}
    \end{subfigure}
    \hfill
    \begin{subfigure}[b]{0.32\textwidth}
        \centering
        \includegraphics[width=\linewidth]{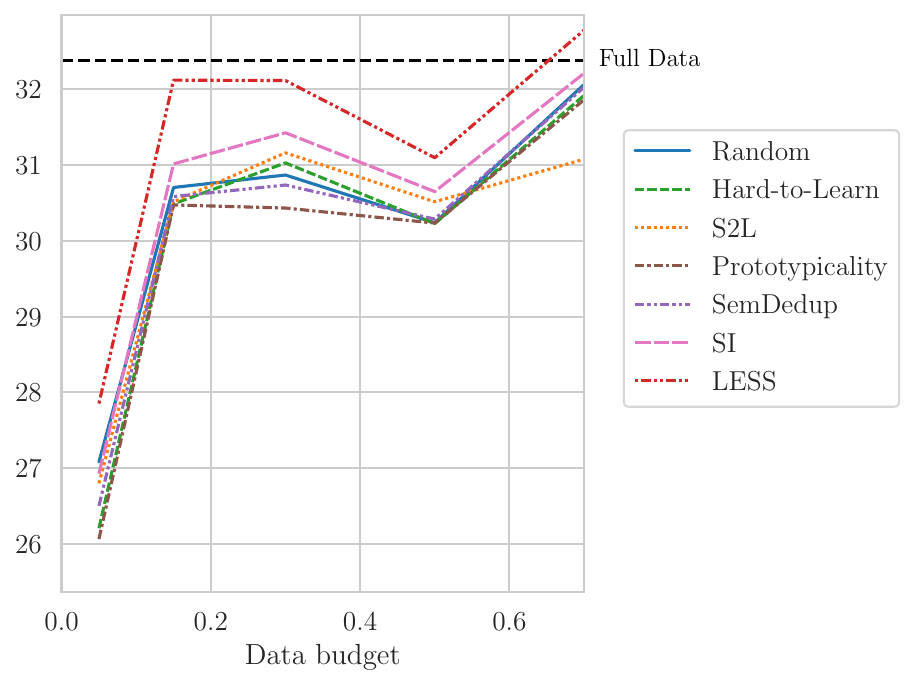}
        \caption{Rouge-L scores of OPT-125M on DialogSum}
        \label{fig:retrain-dialogsum-opt-125m}
    \end{subfigure}
    \par\bigskip
    \begin{subfigure}[b]{0.32\textwidth}
        \centering
        \includegraphics[width=\linewidth]{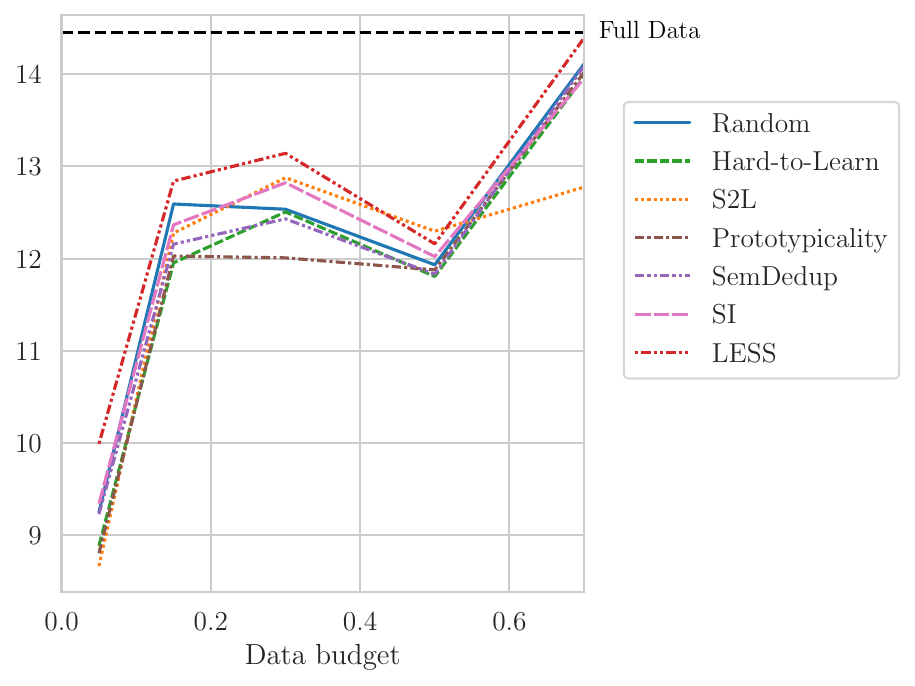}
        \caption{Rouge-2 scores of OPT-125M on DialogSum}
        \label{fig:retrain-dialogsum-opt-125m-rouge-2}
    \end{subfigure}
    \hfill
    \begin{subfigure}[b]{0.32\textwidth}
        \centering
        \includegraphics[width=\linewidth]{figures/iclr/retrain/existing/dialogsum-dialogsum-opt-125m-rouge-2.pdf}
        \caption{Rouge-2 scores of OPT-350M on DialogSum}
        \label{fig:retrain-dialogsum-opt-350m-rouge-2}
    \end{subfigure}
    \hfill
    \begin{subfigure}[b]{0.32\textwidth}
        \centering
        \includegraphics[width=\linewidth]{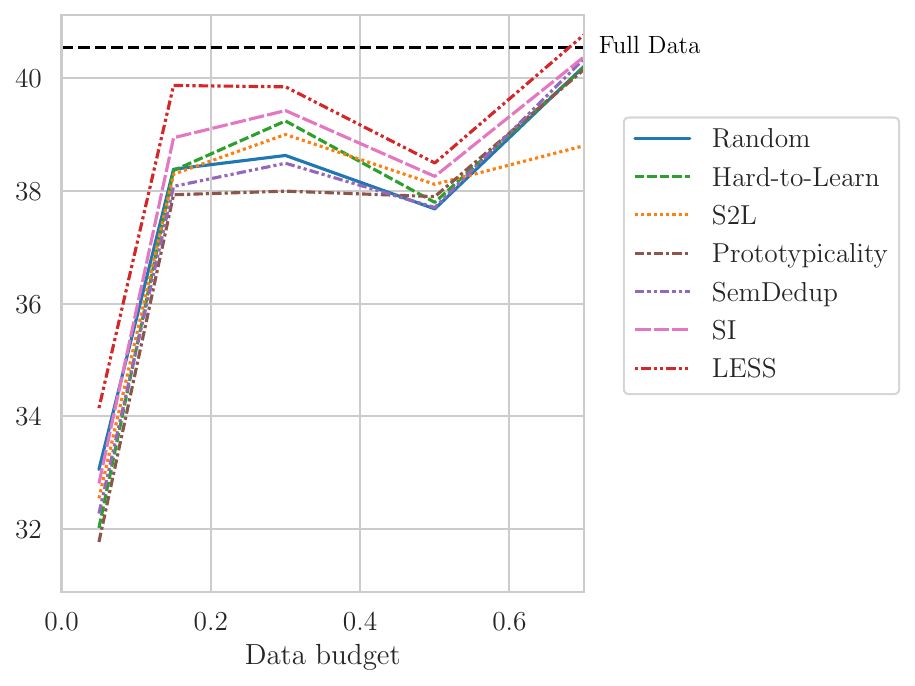}
        \caption{Rouge-1 scores of OPT-125M on DialogSum}
        \label{fig:retrain-dialogsum-opt-125m-rouge-1}
    \end{subfigure}
    \par\bigskip
    \begin{subfigure}[b]{0.32\textwidth}
        \centering
        \includegraphics[width=\linewidth]{figures/iclr/retrain/existing/dialogsum-dialogsum-opt-125m-rouge-1.pdf}
        \caption{Rouge-1 scores of OPT-350M on DialogSum}
        \label{fig:retrain-dialogsum-opt-350m-rouge-1}
    \end{subfigure}
    
    \caption{
        Model performance under different data budgets. 
        }
    \label{fig:retrain_add}
\end{figure*}
\FloatBarrier

\begin{figure*}[ht]
    \centering
    \begin{subfigure}[t]{0.32\textwidth}
        \centering
        \includegraphics[width=\linewidth]{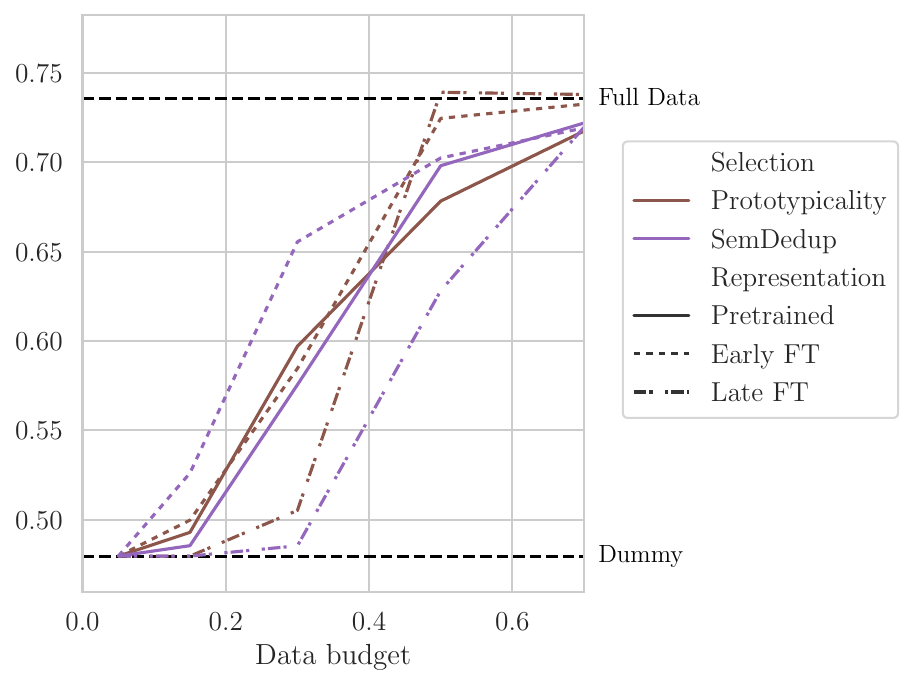}
        \caption{Pretrained vs. Fine-Tuned (FT) Hidden States: $\text{DeBERTaV3}_{\text{Base}}$ on CAD}
        \label{fig:ft-hiddenstate-cad-debertav3-base}
    \end{subfigure}
    \hfill
    \begin{subfigure}[t]{0.32\textwidth}
        \centering
        \includegraphics[width=\linewidth]{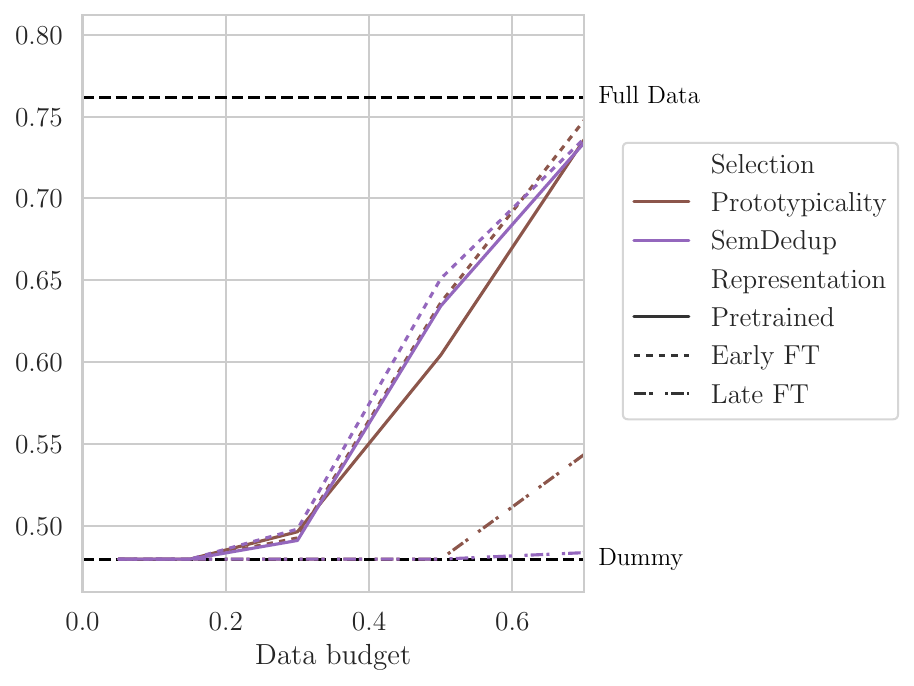}
        \caption{Pretrained vs. Fine-Tuned (FT) Hidden States: $\text{DeBERTaV3}_{\text{Large}}$ on CAD}
        \label{fig:ft-hiddenstate-cad-debertav3-large}
    \end{subfigure}
    \hfill
    \begin{subfigure}[t]{0.32\textwidth}
        \centering
        \includegraphics[width=\linewidth]{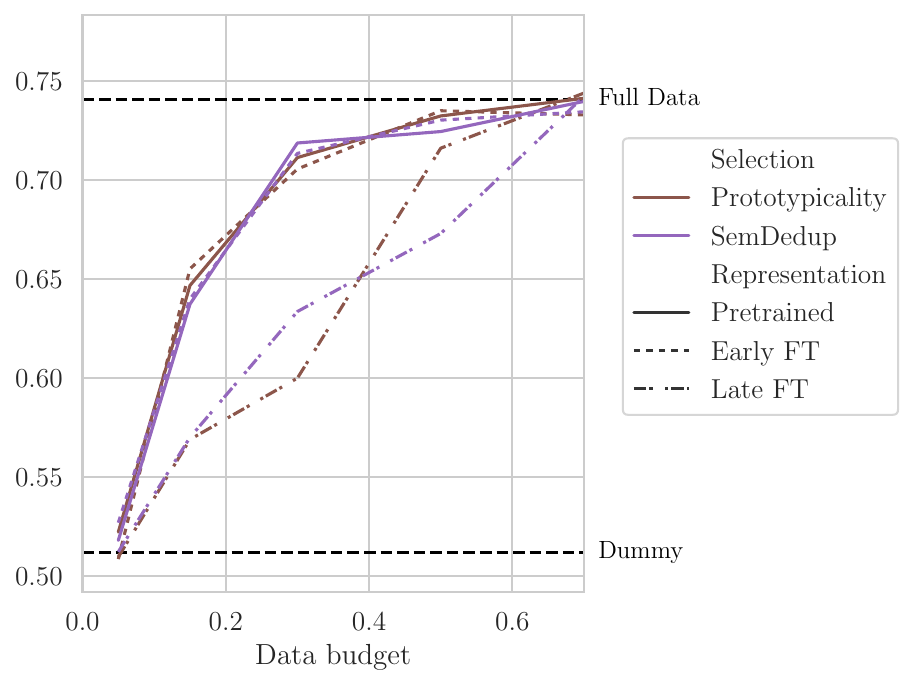}
        \caption{Pretrained vs. Fine-Tuned (FT) Hidden States: $\text{DeBERTaV3}_{\text{Base}}$ on WinoGrande}
        \label{fig:ft-hiddenstate-winogrande-debertav3-base}
    \end{subfigure}
    \par\bigskip
    \begin{subfigure}[t]{0.32\textwidth}
        \centering
        \includegraphics[width=\linewidth]{figures/iclr/retrain/fthiddenstate/winogrande-winogrande-debertav3-large-accuracy.pdf}
        \caption{Pretrained vs. Fine-Tuned (FT) Hidden States: $\text{DeBERTaV3}_{\text{Large}}$ on WinoGrande}
        \label{fig:ft-hiddenstate-winogrande-debertav3-large}
    \end{subfigure}
    \hfill
    \begin{subfigure}[t]{0.32\textwidth}
        \centering
        \includegraphics[width=\linewidth]{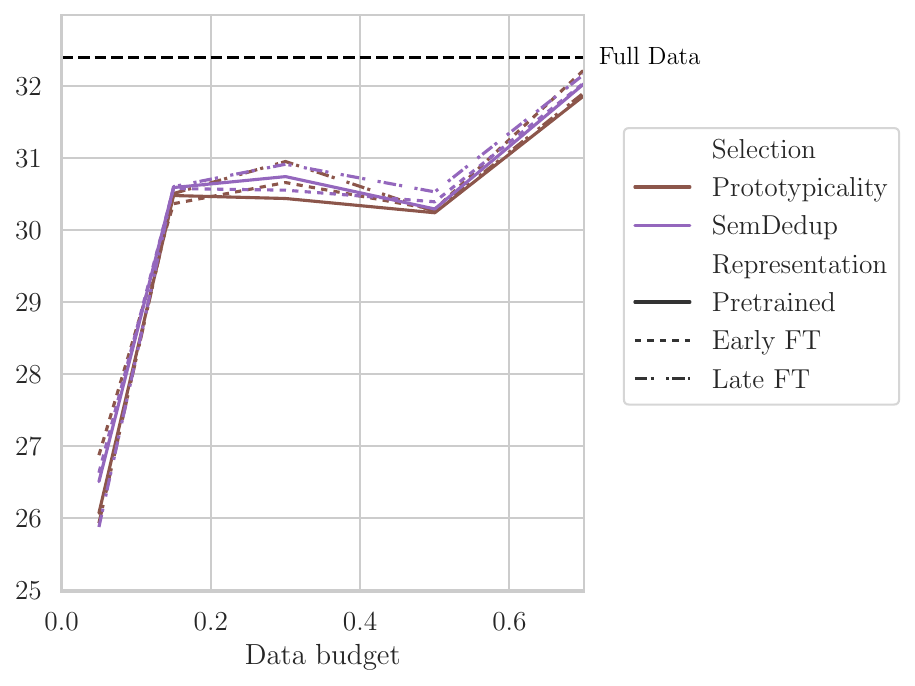}
        \caption{Pretrained vs. Fine-Tuned (FT) Hidden States: OPT-125M on DialogSum (Rouge-L)}
        \label{fig:ft-hiddenstate-dialogsum-opt-125m-rouge-l}
    \end{subfigure}
    \hfill
    \begin{subfigure}[t]{0.32\textwidth}
        \centering
        \includegraphics[width=\linewidth]{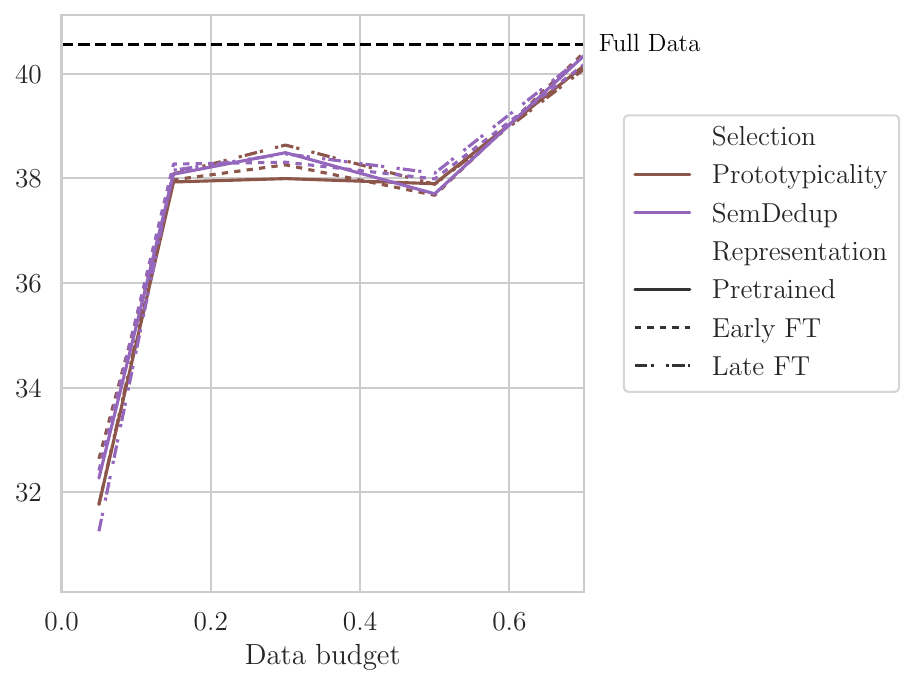}
        \caption{Pretrained vs. Fine-Tuned (FT) Hidden States: OPT-125M on DialogSum (Rouge-1)}
        \label{fig:ft-hiddenstate-dialogsum-opt-125m-rouge-1}
    \end{subfigure}
    \par\bigskip
    \begin{subfigure}[t]{0.32\textwidth}
        \centering
        \includegraphics[width=\linewidth]{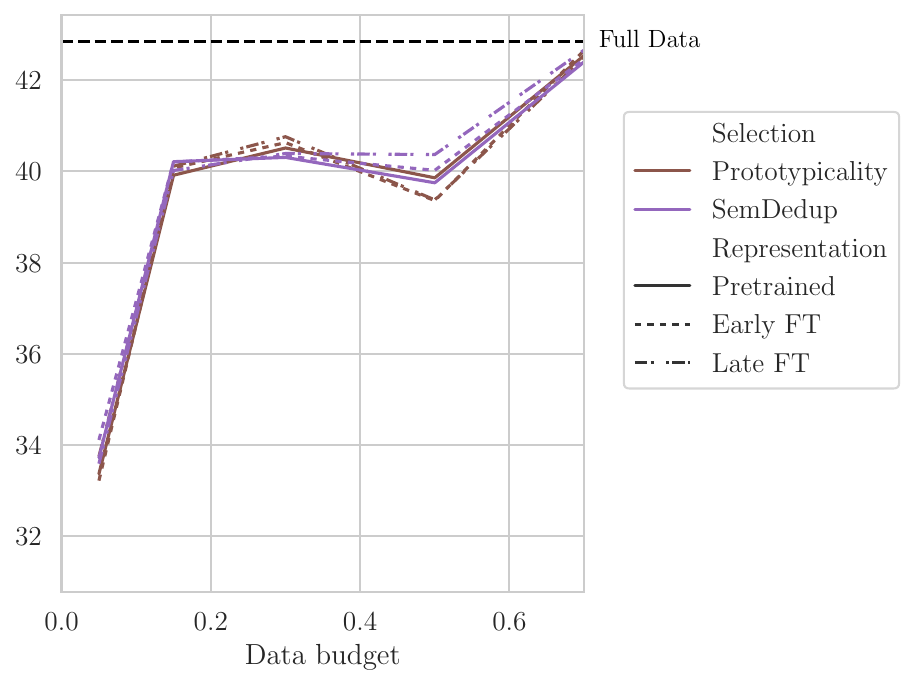}
        \caption{Pretrained vs. Fine-Tuned (FT) Hidden States: OPT-350M on DialogSum (Rouge-1)}
        \label{fig:ft-hiddenstate-dialogsum-opt-350m-rouge-1}
    \end{subfigure}
    \hfill
    \begin{subfigure}[t]{0.32\textwidth}
        \centering
        \includegraphics[width=\linewidth]{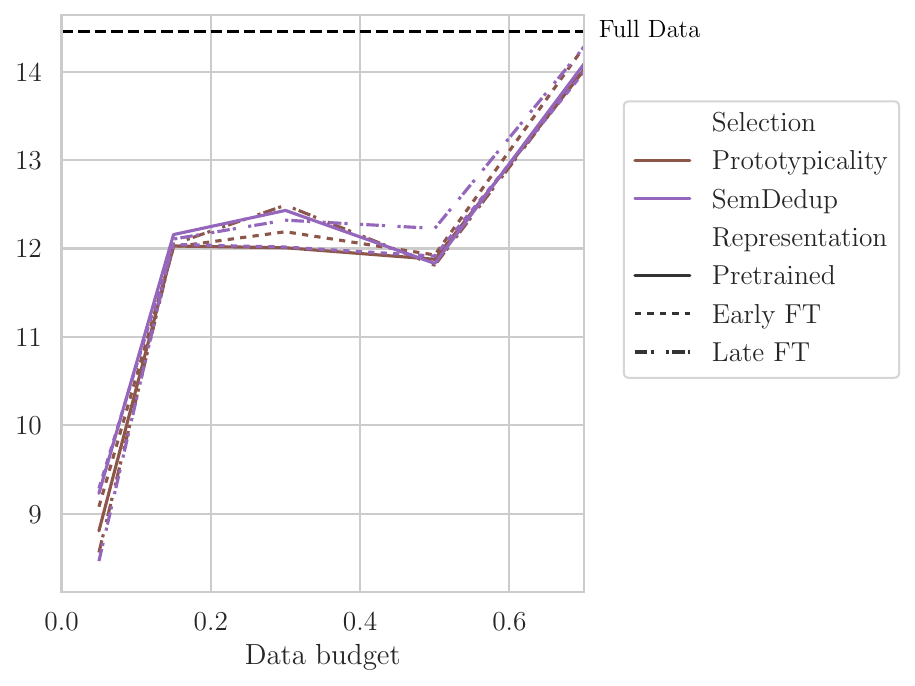}
        \caption{Pretrained vs. Fine-Tuned (FT) Hidden States: OPT-125M on DialogSum (Rouge-2)}
        \label{fig:ft-hiddenstate-dialogsum-opt-125m-rouge-2}
    \end{subfigure}
    \hfill
    \begin{subfigure}[t]{0.32\textwidth}
        \centering
        \includegraphics[width=\linewidth]{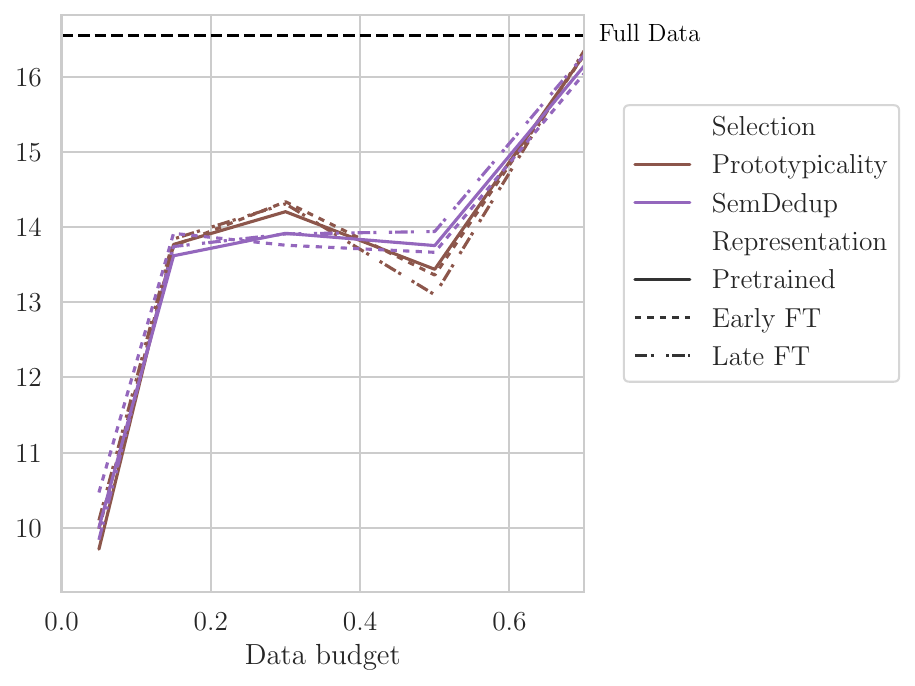}
        \caption{Pretrained vs. Fine-Tuned (FT) Hidden States: OPT-350M on DialogSum (Rouge-2)}
        \label{fig:ft-hiddenstate-dialogsum-opt-350m-rouge-2}
    \end{subfigure}
    \caption{
        Ablation studies on pretrained vs. fine-tuned hidden states. 
        }
    \label{fig:ablation_add_ft_hiddenstate}
\end{figure*}
\FloatBarrier

\begin{figure*}[h]
    \centering
    \begin{subfigure}[t]{0.32\textwidth}
        \centering
        \includegraphics[width=\linewidth]{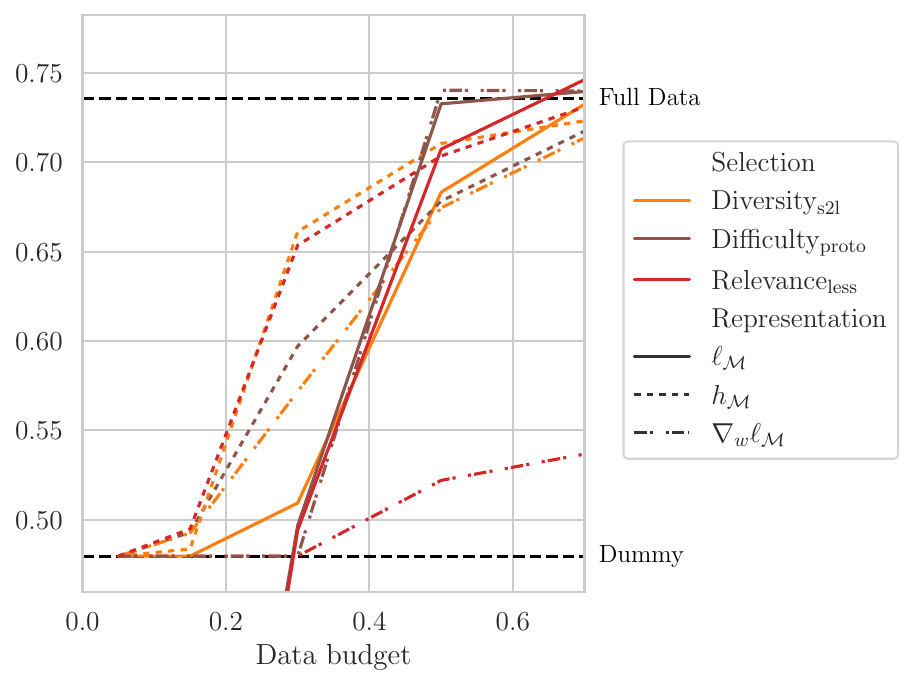}
        \caption{$\text{DeBERTaV3}_{\text{Base}}$ on CAD (F1)}
        \label{fig:sampling-feature-cad-debertav3-base}
    \end{subfigure}
    \hfill
    \begin{subfigure}[t]{0.32\textwidth}
        \centering
        \includegraphics[width=\linewidth]{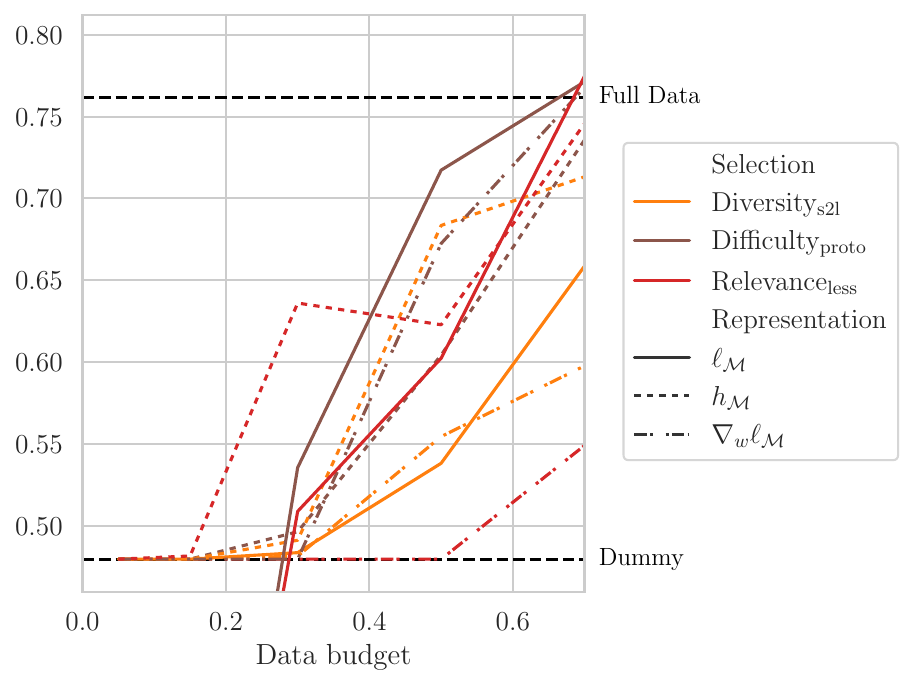}
        \caption{$\text{DeBERTaV3}_{\text{Large}}$ on CAD (F1)}
        \label{fig:sampling-feature-cad-debertav3-large}
    \end{subfigure}
    \hfill
    \begin{subfigure}[t]{0.32\textwidth}
        \centering
        \includegraphics[width=\linewidth]{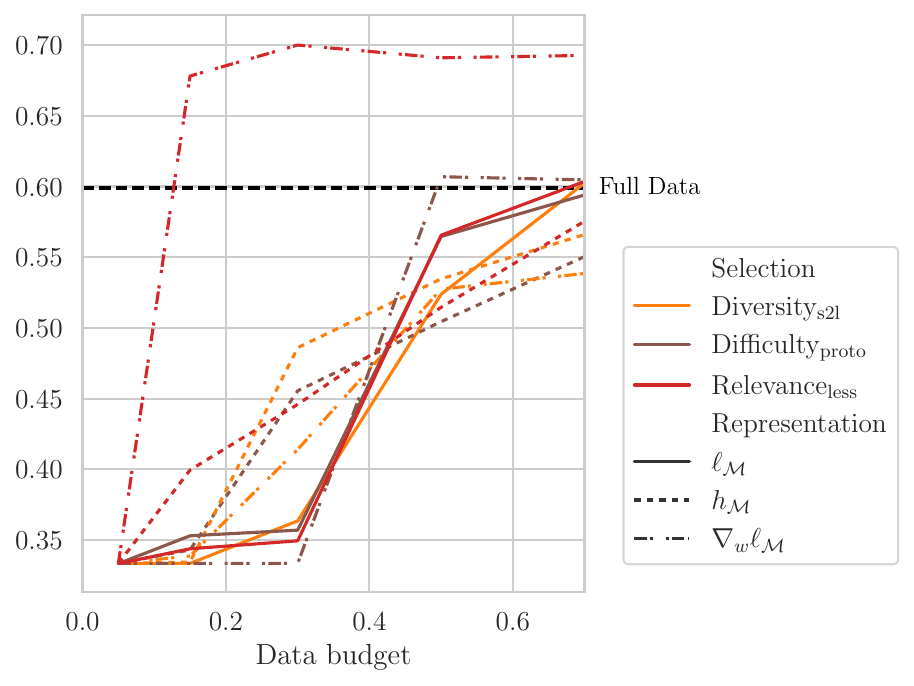}
        \caption{$\text{DeBERTaV3}_{\text{Base}}$ on DynaHate (F1)}
        \label{fig:sampling-feature-dynahate-debertav3-base}
    \end{subfigure}
    \par\bigskip
    \begin{subfigure}[t]{0.32\textwidth}
        \centering
        \includegraphics[width=\linewidth]{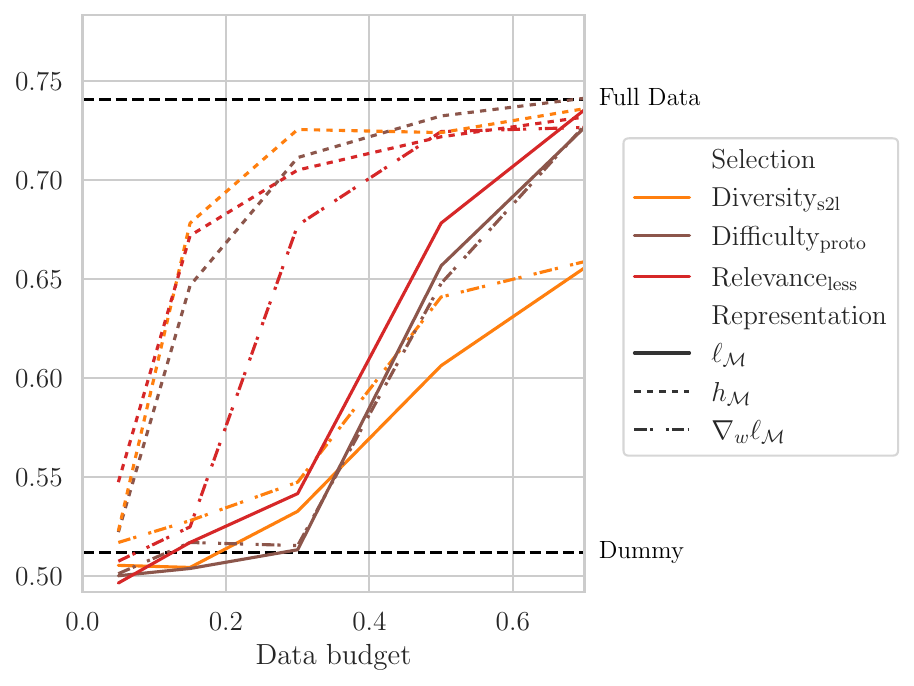}
        \caption{$\text{DeBERTaV3}_{\text{Base}}$ on WinoGrande (Accuracy)}
        \label{fig:sampling-feature-winogrande-debertav3-base}
    \end{subfigure}
    \hfill
    \begin{subfigure}[t]{0.32\textwidth}
        \centering
        \includegraphics[width=\linewidth]{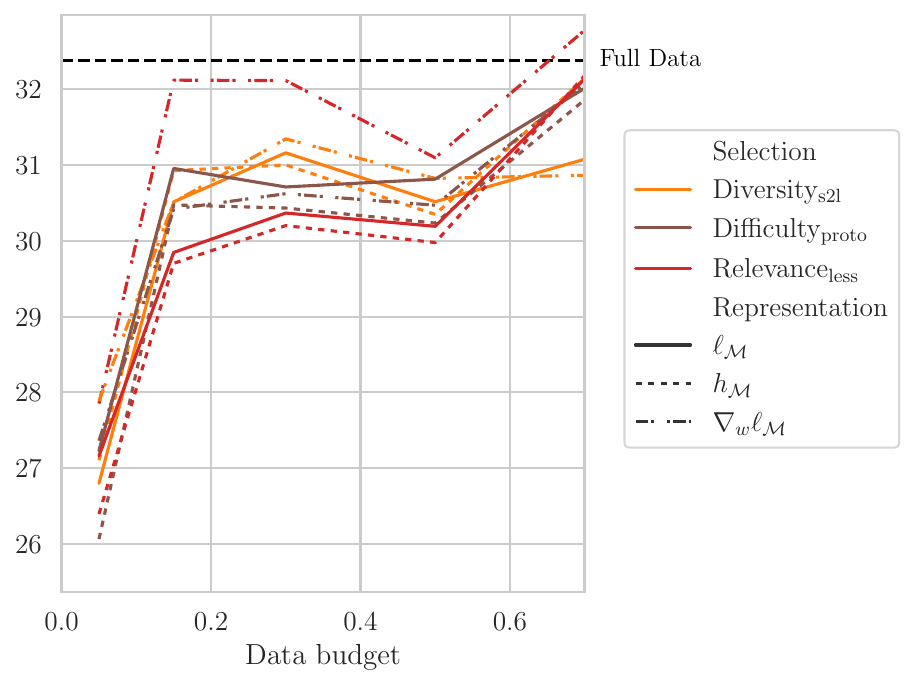}
        \caption{OPT-125M on DialogSum (Rouge-L)}
        \label{fig:sampling-feature-dialogsum-opt-125m-rouge-l}
    \end{subfigure}
    \hfill
    \begin{subfigure}[t]{0.32\textwidth}
        \centering
        \includegraphics[width=\linewidth]{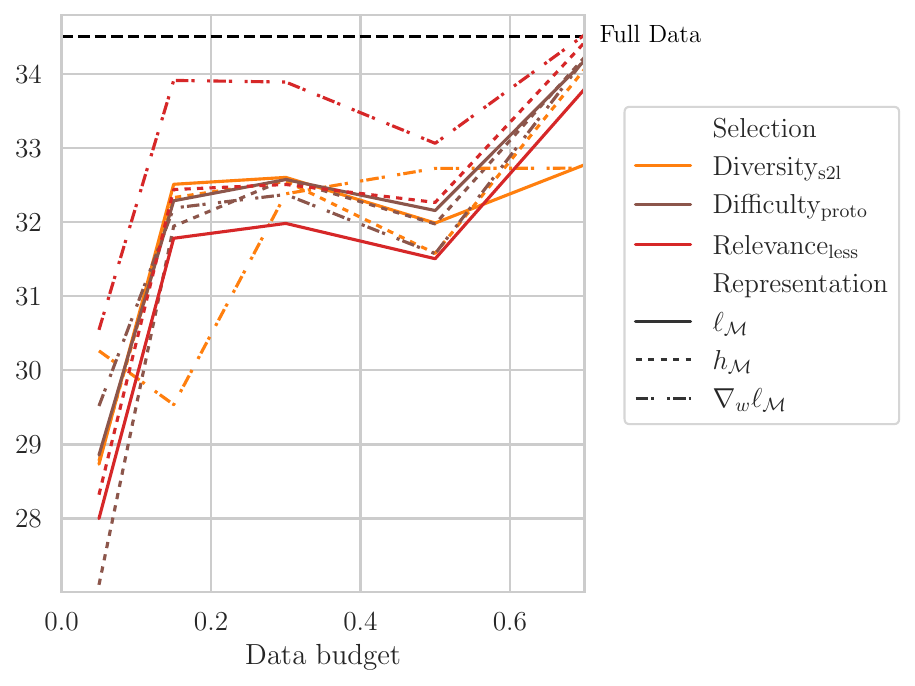}
        \caption{OPT-350M on DialogSum (Rouge-L)}
        \label{fig:sampling-feature-dialogsum-opt-350m-rouge-l}
    \end{subfigure}
    \par\bigskip
    \begin{subfigure}[t]{0.32\textwidth}
        \centering
        \includegraphics[width=\linewidth]{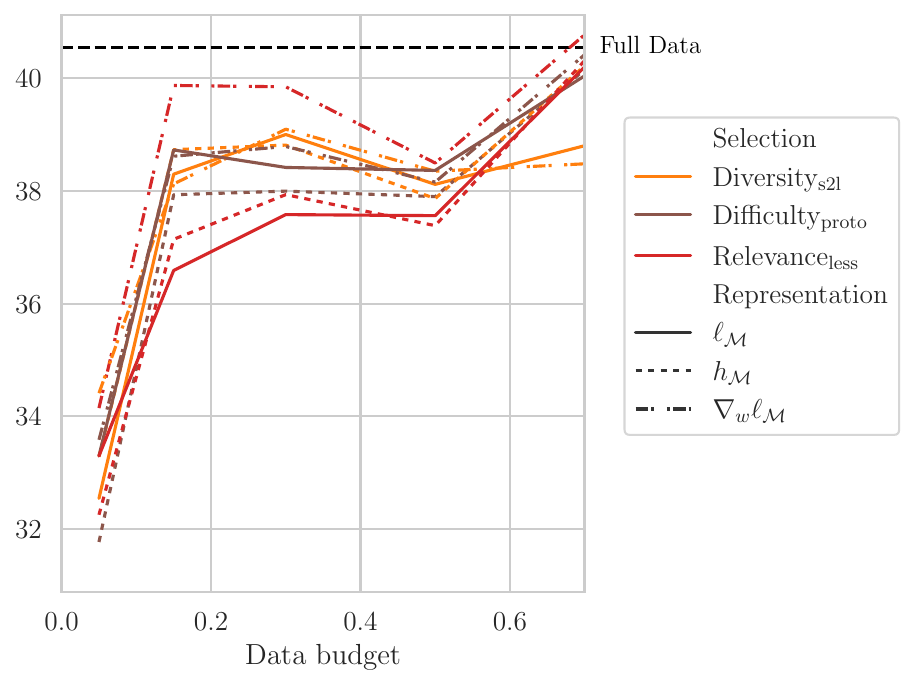}
        \caption{OPT-125M on DialogSum (Rouge-1)}
        \label{fig:sampling-feature-dialogsum-opt-125m-rouge-1}
    \end{subfigure}
    \hfill
    \begin{subfigure}[t]{0.32\textwidth}
        \centering
        \includegraphics[width=\linewidth]{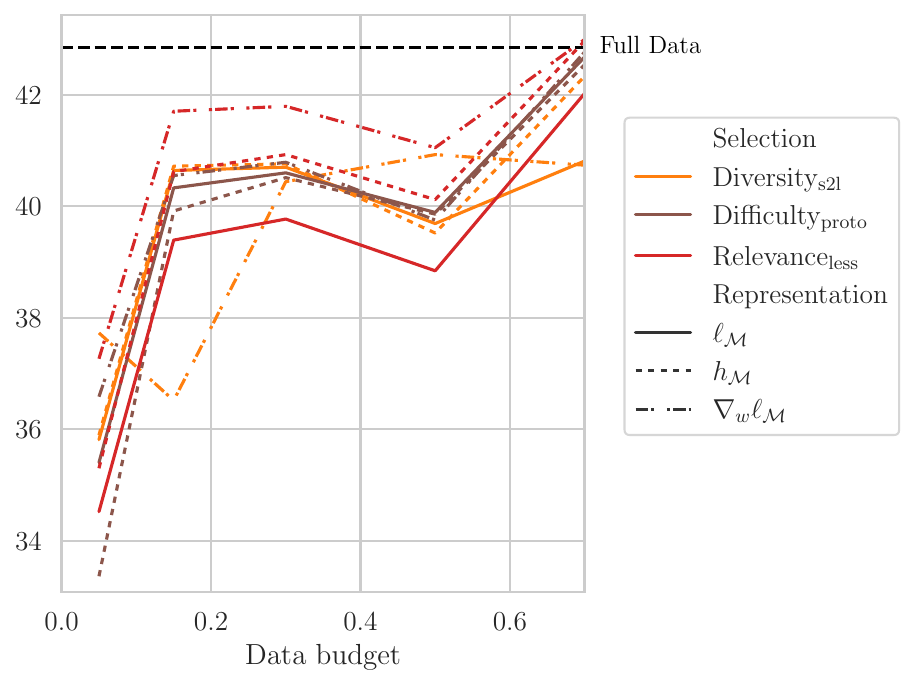}
        \caption{OPT-350M on DialogSum (Rouge-1)}
        \label{fig:sampling-feature-dialogsum-opt-350m-rouge-1}
    \end{subfigure}
    \hfill
    \begin{subfigure}[t]{0.32\textwidth}
        \centering
        \includegraphics[width=\linewidth]{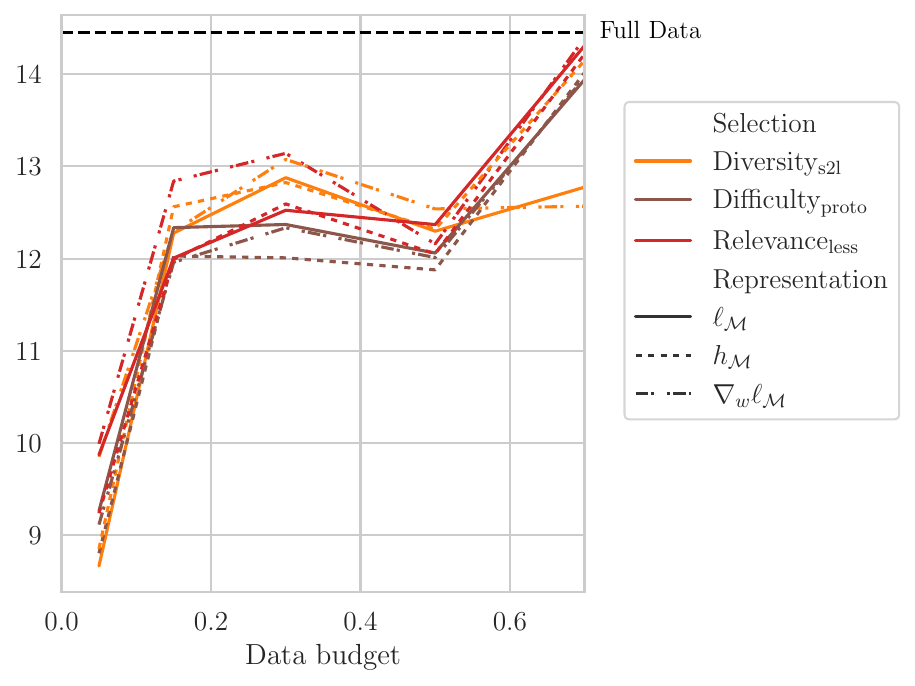}
        \caption{OPT-125M on DialogSum (Rouge-2)}
        \label{fig:sampling-feature-dialogsum-opt-125m-rouge-2}
    \end{subfigure}
    \par\bigskip
    \begin{subfigure}[t]{0.32\textwidth}
        \centering
        \includegraphics[width=\linewidth]{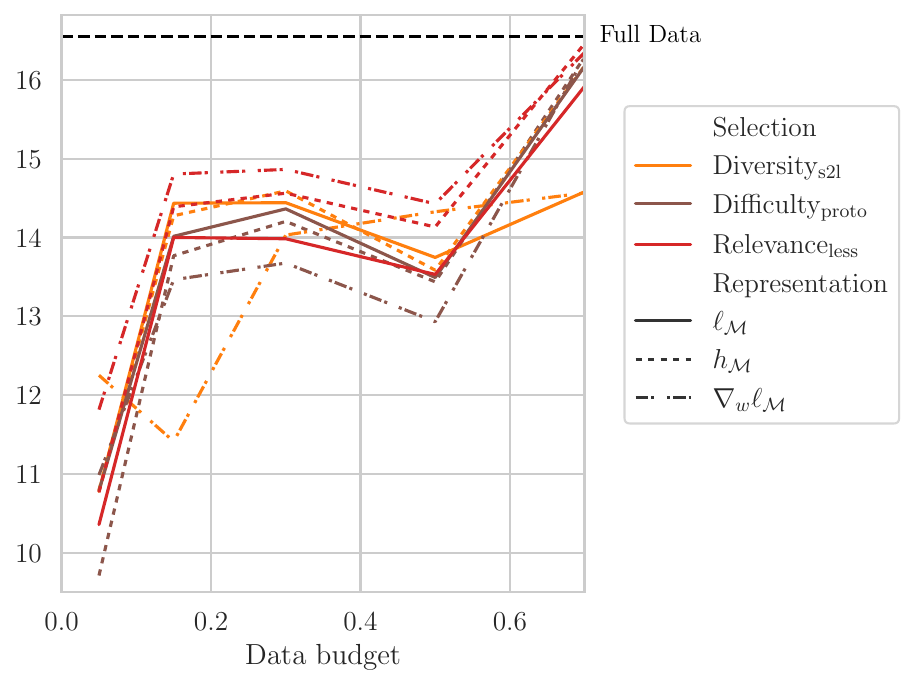}
        \caption{OPT-350M on DialogSum (Rouge-2)}
        \label{fig:sampling-feature-dialogsum-opt-350m-rouge-2}
    \end{subfigure}
    \hfill
    \caption{
        Ablation studies on using different data representations with the same selection algorithm.
    }
    \label{fig:sampling-feature-add}

\end{figure*}
\FloatBarrier

\end{document}